\documentclass[10pt,journal,compsoc]{IEEEtran}

\usepackage{times}
\usepackage{epsfig}
\usepackage{graphicx}
\usepackage{amsmath}
\usepackage{amssymb}
\newcommand{\tabincell}[2]{\begin{tabular}{@{}#1@{}}#2\end{tabular}}
\usepackage{multirow}
\usepackage{booktabs}
\usepackage{tabu}
\usepackage[dvipsnames]{xcolor}
\usepackage{spreadtab}
\usepackage{tabularx}

\usepackage{hyperref}
\usepackage[switch]{lineno} 
\usepackage{amsthm}
\newtheorem{theorem}{Theorem}[section]
\newtheorem{lemma}[theorem]{Lemma}
\theoremstyle{definition}
\newtheorem{definition}{Definition}[section]
\newtheorem*{remark}{Remark}

\newcommand{\etal}{\textit{et al.}}
\ifCLASSOPTIONcompsoc
  \usepackage[nocompress]{cite}
\else
  \usepackage{cite}
\fi

\ifCLASSINFOpdf
\else
\fi

\begin{document}
%
\title{PRIN/SPRIN: On Extracting Point-wise Rotation Invariant Features}
%
%
%
%
\author{Yang You, Yujing Lou, Ruoxi Shi, Qi Liu, Yu-Wing Tai, \\
Lizhuang Ma, Weiming Wang, Cewu Lu~\IEEEmembership{Member,~IEEE}
\IEEEcompsocitemizethanks{\IEEEcompsocthanksitem Yang You, Yujing Lou, Ruoxi Shi, Qi Liu, Lizhuang Ma, Weiming Wang, Cewu Lu are with Shanghai Jiao Tong University, Shanghai, 200240, China. Yu-Wing Tai is with the Department of Computer Science and Engineering, The Hong Kong University of Science and Technology, Hong Kong. Cewu Lu is also the member of Qing Yuan Research Institute, Shanghai Qizhi Research Institute and MoE Key Lab of Artificial Intelligence, AI Institute, Shanghai Jiao Tong University, China.
\\
\IEEEcompsocthanksitem Weiming Wang and Cewu Lu are the corresponding authors.\protect\\
E-mail: wangweiming@sjtu.edu.cn, lucewu@sjtu.edu.cn.
}
}

%
%

\markboth{Journal of \LaTeX\ Class Files,~Vol.~14, No.~8, August~2015}%
{Shell \MakeLowercase{\textit{et al.}}: Bare Advanced Demo of IEEEtran.cls for IEEE Computer Society Journals}
%



\IEEEtitleabstractindextext{%
\begin{abstract}
Point cloud analysis without pose priors is very challenging in real applications, as the orientations of point clouds are often unknown. In this paper, we propose a brand new point-set learning framework PRIN, namely, \textbf{P}oint-wise \textbf{R}otation \textbf{I}nvariant \textbf{N}etwork, focusing on rotation invariant feature extraction in point clouds analysis. We construct spherical signals by Density Aware Adaptive Sampling to deal with distorted point distributions in spherical space. Spherical Voxel Convolution and Point Re-sampling are proposed to extract rotation invariant features for each point. In addition, we extend PRIN to a sparse version called SPRIN, which directly operates on sparse point clouds. Both PRIN and SPRIN can be applied to tasks ranging from object classification, part segmentation, to 3D feature matching and label alignment. Results show that, on the dataset with randomly rotated point clouds, SPRIN demonstrates better performance than state-of-the-art methods without any data augmentation. We also provide thorough theoretical proof and analysis for point-wise rotation invariance achieved by our methods. The code to reproduce our results will be made publicly available.
\end{abstract}

\begin{IEEEkeywords}
Point cloud, object analysis, rotation invariance, feature learning
\end{IEEEkeywords}}

\maketitle

\IEEEdisplaynontitleabstractindextext


%
\IEEEpeerreviewmaketitle

\ifCLASSOPTIONcompsoc
\IEEEraisesectionheading{\section{Introduction}\label{sec:introduction}}
\else
\section{Introduction}
\label{sec:introduction}
\fi
Deep learning on point clouds has received tremendous interest in recent years. Since depth cameras capture point clouds directly, efficient and robust point processing methods like classification, segmentation and reconstruction have become key components in real-world applications. Robots, autonomous cars, 3D face recognition and many other fields rely on learning and analysis of point clouds. 

Existing works like PointNet~\cite{8099499} and PointNet++~\cite{qi2017pointnet++} have achieved remarkable results in point cloud learning and shape analysis. But they focus on objects with canonical orientations. In real applications, these methods fail to be applied to rotated shape analysis since the model orientation is often unknown as a priori, as shown in Figure~\ref{fig:pn_partition}. In addition, existing frameworks require massive data augmentation to handle rotations, which induces unacceptable computational cost.


\begin{figure}[h]
\begin{center}
  \includegraphics[width=\linewidth]{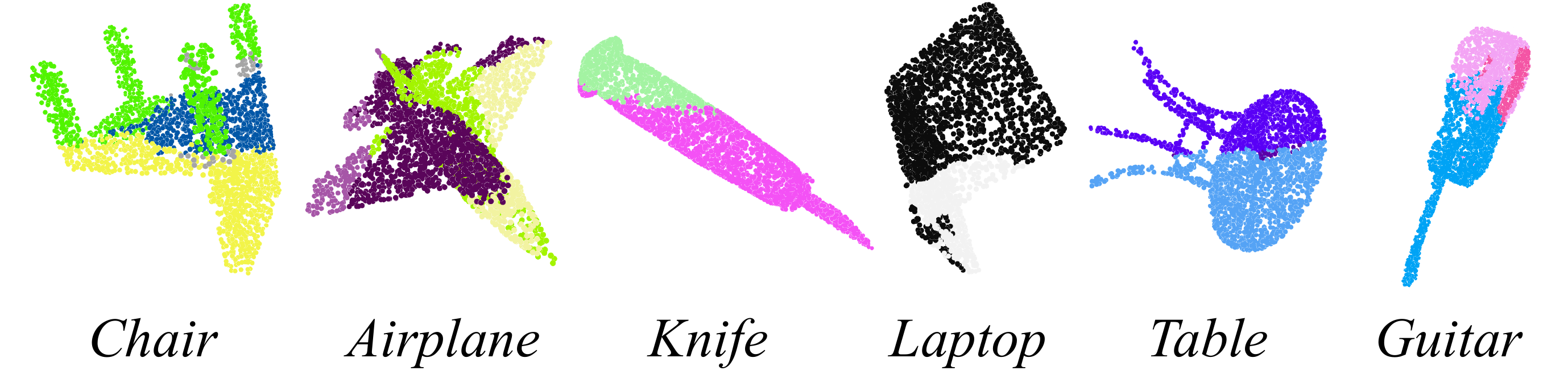}
\end{center}
\caption{\textbf{PointNet++ part segmentation results on rotated shapes.} When trained on objects with canonical orientations and evaluated on rotated ones, PointNet++ is unaware of their orientations and fails to segment their parts out.}
\label{fig:pn_partition}
\end{figure}


Spherical CNN~\cite{cohen2018spherical} and a similar method~\cite{esteves2018learning} try to solve this problem and propose a global feature extracted from continuous meshes, while they are not suitable for point clouds since they project 3D meshes onto their enclosing spheres using a ray casting scheme. Difficulty lies in how to apply spherical convolution in continuous domain to sparse point clouds. Besides, by projecting onto a unit sphere, their method is limited to processing convex shapes, ignoring any concave structures. Therefore, we propose a point-wise rotation invariant network (PRIN) to handle these problems. Firstly, we observe the discrepancy between unit spherical space and Euclidean space, and propose Density Aware Adaptive Sampling (DAAS) to avoid biases. Secondly, we come up with Spherical Voxel Convolution (SVC) without loss of rotation invariance, which is able to capture any concave information. Furthermore, we propose Point Re-sampling module that helps to extract rotation invariant features for each point.



PRIN is a network that directly takes point clouds with random rotations as the input, and predicts both categories and point-wise segmentation labels without data augmentation. It absorbs the advantages of both Spherical CNN and PointNet-like network by keeping rotation invariant features, while maintaining a one-to-one point correspondence between the input and output.
PRIN learns rotation invariant features at spherical voxel grids. Afterwards, these features could be aggregated into a global descriptor or per-point descriptor to conduct model classification or part segmentation, respectively. We rigorously prove the point-wise rotation invariance of PRIN under certain conditions. 

In addition, we extend our PRIN framework and propose a sparse version of PRIN called SPRIN. SPRIN considers the input as the Dirac delta function and gives rotation invariant features when the filter is constant on the left coset of $z$-axis rotation.

We experimentally compare PRIN and SPRIN with various state-of-the-art approaches on the benchmark dataset: ShapeNet part dataset~\cite{Yi16} and ModelNet40~\cite{wu20153d}. Additionally, both PRIN and SPRIN can be applied to 3D point matching and label alignment. Both PRIN and SPRIN exhibit remarkable performance. SPRIN also achieves the state-of-the-art performance on part segmentation.

\begin{figure*}[t]
\begin{center}
\includegraphics[width=\linewidth]{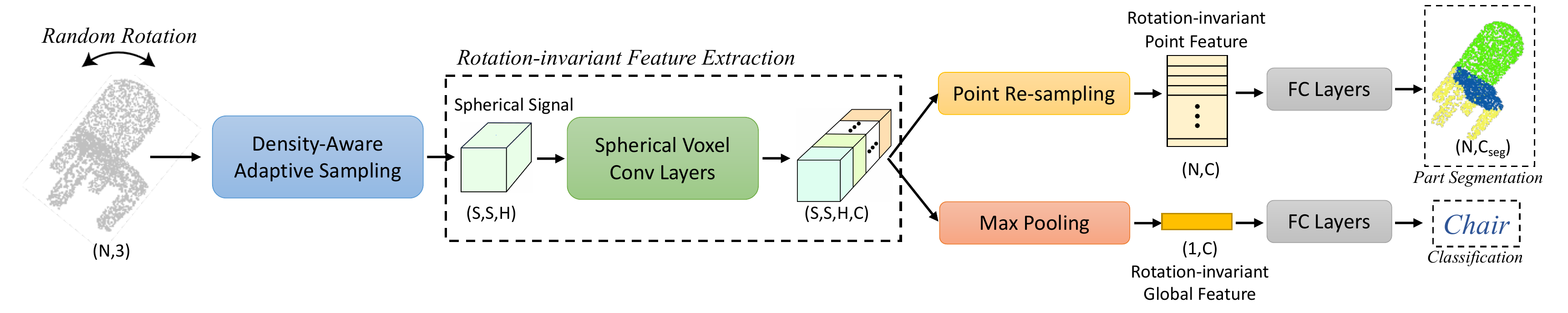}
\end{center}
\caption{\textbf{PRIN Architecture.}  
Our network takes sparse points as input, and then uses Density-Aware Adaptive Sampling to transform the signals into spherical voxel grids. The spherical voxel signals are then passed through several Spherical Voxel Convolution layers, ending with a feature at each spherical voxel grid. Any point feature can be extracted by point re-sampling, which is used to do point-wise part segmentation. All these voxel features can also be max-pooled to get a global feature, which is suitable for classification.
%
}
\label{fig:pipeline}
\end{figure*}

The key contributions of this paper are as follows:
\begin{itemize}
\item{We design two novel deep network processing pipelines PRIN/SPRIN that extracts rotation invariant point-level features.}
\item{Three key modules: Density Aware  Adaptive  Sampling (DAAS), Spherical Voxel Convolution (SVC) and Point Re-sampling are proposed for PRIN. }
\item{We propose a special spherical voxel convolution and prove that it is rotation equivariant. In addition, we extend this convolution to the sparse domain and develop a sparse version of PRIN called SPRIN. Rigorous proof of point-wise rotation invariance is given for both PRIN and SPRIN frameworks.}
\item{We show that PRIN/SPRIN can be used for point cloud part segmentation, classification, 3D point matching and label alignment under different rotations. SPRIN achieves the state-of-the-art performance on part segmentation.}
\end{itemize}

A  preliminary  version  of  this  work  was  presented  in  AAAI2020~\cite{you2020pointwise}.  In  this  study,  we  extend  it  in  two  fundamental aspects. First, we provide thorough theoretical analysis and proof for the point-wise rotation invariance in PRIN, and some necessary filter conditions are proposed. With this revised version of spherical voxel convolution, we achieve much better results than our previous work. Second, we extend PRIN to sparse domain, where the Dirac delta function is leveraged and sparse correlation is proposed with guaranteed rotation invariance. The sparse version of PRIN, which is named SPRIN, achieves the state-of-the-art performance on ShapeNet part segmentation.


\section{Related Work}

\subsection{Rotation invariant Features}
Rotation Invariance is often regarded as a preliminary to the success of template matching and object detection, in both 2D and 3D domains.

The development of rotation invariant features from geometries could be retrospected to manual designed features, including Structural Indexing~\cite{stein1992structural}, Signature of Histogram Orientations (SHOT)~\cite{shot}, CGF~\cite{khoury2017learning} and Snapshots~\cite{malassiotis2007snapshots}. They construct a local reference frame (LRF) which aligns the model into its canonical pose in order to extract rotation invariant point features. However, these methods depend on local surface variations, therefore are sensitive to noise and point densities. Besides, these descriptors rely on delicate hand-craft design, and could only capture low-level geometric features. For a more complete review on traditional feature descriptors, we refer to Guo \etal~\cite{guo20143d}.


Recently, some papers consider generalizations of 2D CNNs that exploit larger groups of symmetries~\cite{gens2014deep,cohen2016steerable}, including the 2D/3D rotation group~\cite{dieleman2015rotation}. Spherical CNN~\cite{cohen2018spherical} and a similar method~\cite{esteves2018learning} propose to extract global rotation invariant features from continuous meshes, while they are not suitable for point clouds since they project 3D meshes onto their enclosing spheres using a ray casting scheme.

In parallel to group invariant/equivalent convolutions, some researchers incorporate rotation invariant point-level convolutions. SRINet~\cite{sun2019srinet} proposes the point projection feature, which is invariant to the rotation of the input point cloud. It introduces an efficient key point descriptor to assign each point with different response and help recognize the overall geometry. Poulenard \etal~\cite{poulenard2019effective} employ a spherical harmonics based kernel at different layers of a point-based PCNN architecture.  Kim \etal~\cite{kim2020rotation} take advantages of multi-level abstraction based on graph convolutional neural networks, which constructs a descriptor hierarchy to encode rotation invariant shape information of an input object in a bottom-up manner. Zhang \etal~\cite{zhang2019rotation} use low-level rotation invariant geometric features such as distances and angles to design a convolution operator for point cloud learning. PPF-FoldNet~\cite{deng2018ppf} obtains unsupervised rotation invariant point-wise features via an auto encoder-decoder structure. Li \etal~\cite{li2020rotation} present a network architecture to embed rotation invariant representations into features, encoding local relations between points and their neighbors, and the global shape structure. Using graphs as point cloud representation is another way to achieve rotation invariance like Zhang \etal~\cite{zhang2018graph} and Wang \etal~\cite{wang2021point}.

\subsection{Rotation Equivariance in Point Clouds}
Rotation Equivariance is closely related to rotation invariant point features, where feature locations are equivalently transformed according to the input rotation. Tensor field networks~\cite{thomas2018tensor} build filters from spherical harmonics and are proven to be equivariant to rotations. SE(3)-Transformers~\cite{fuchs2020se} combine graph networks and self-attention mechanisms to fulfill the goal of rotation-translation equivariance. However, both Tensor field networks and SE(3)-Transformers do not scale well with input points and fail to be applied to large point clouds (e.g. 2048 points). In contrast, our SPRIN can easily fit 2048 points with a large batch size, due to the introduced sparse correlation.

Cohen \etal~\cite{cohen2018spherical}, Esteves \etal~\cite{esteves2018learning} and Cruz~\cite{cruz2012scale} all decomposes SO(3) group with irreducible harmonic orthogonal basis, which is proven to strictly equivariant to SO(3) rotations. Quaternion equivariant capsule networks~\cite{zhao2020quaternion} disentangles geometry from pose with dynamic routing algorithm from capsule networks. 
Rotation Equivariance can be also achieved by giving a rotation invariant representation for each individual point, such as Khoury \etal~\cite{khoury2017learning} and Gojcic \etal~\cite{gojcic2019perfect}. 

\subsection{Deep Learning on 3D Shapes}
As the consequence of success in deep learning, various methods have been proposed for better understanding 3D models. Convolutional neural networks are applied to volumetric data since its format is similar to pixels in an image and easy to transfer to existing frameworks. 3D ShapeNet~\cite{wu20153d} and VoxNet~\cite{7353481} are pioneers introducing fully-connected networks to voxels. However, dealing with voxel data requires large memory and the sparsity of point sets also makes it challenging to extract particular features from big data. As a consequence, MinkowskiNet~\cite{choy20194d} and Submanifold Sparse Convolutional Networks~\cite{graham2017submanifold} try to solve this by conducting sparse convolutions near existing points. Our SPRIN also takes inspiration from these works to improve the scalability of PRIN.

Another research branch is multi-view methods. MVCNN~\cite{su2015multi} renders 3D models into multi-view images and propagates these images into traditional convolutional neural networks. 
These approaches are limited to simple tasks like classification and not suitable for 3D part segmentation, key point matching or other tasks. 

Dealing with point clouds directly is another popular branch, among which PointNet~\cite{8099499} is the pioneer in building a general framework for learning point clouds. Since then, many networks are proposed to learn from point clouds. PointNet++~\cite{qi2017pointnet++} extends PointNet by introducing a hierarchical structure. RSNet~\cite{huang2018recurrent} combines a novel slice pooling layer, Recurrent Neural Network (RNN) layers, and a slice unpooling layer, to better propagate features among different parts. SplatNet~\cite{su18splatnet} uses sparse bilateral convolutional layers to enable hierarchical and spatially-aware feature learning. 
DGCNN~\cite{dgcnn} proposes EdgeConv, which explicitly constructs a local graph and learns the embeddings for the edges in both Euclidean and semantic spaces.
As a follow-up, LDGCNN~\cite{zhang2019linked} links the hierarchical features from different dynamic graphs based on DGCNN.
SO-Net~\cite{li2018so} models the spatial distribution of point cloud by building a Self-Organizing Map (SOM) and its receptive field can be systematically adjusted
by conducting point-to-node k nearest neighbor search.
SpiderCNN~\cite{xu2018spidercnn} designs the convolutional filter as a product of a simple step function that captures local geodesic information and a Taylor polynomial that ensures
the expressiveness.
DeepSets~\cite{zaheer2017deep} provides a family of functions which are permutation-invariant and gives a competitive result on point clouds.
Point2Sequence~\cite{liu2019point2sequence} uses a recurrent neural network (RNN)
based encoder-decoder structure, where an attention mechanism is proposed to highlight the importance of different area scales.
Kd-Network~\cite{8237361} utilizes kd-tree structures to form the computational graph, which learns from point clouds hierarchically. However, few of them are robust to random rotations, making it hard to apply them to real applications.

\section{Preliminaries}
\subsection{Unit Sphere: $S^2$}
A point $s$ in a two-dimensional sphere can be uniquely described by its azimuthal and polar angles: $(\alpha, \beta)$, where $\alpha \in [0, 2\pi], \beta \in [0, \pi]$. Furthermore, if we denote $n = (0, 0, 1)^T$ as the North Pole, the coordinate of $s$ is given by $Z(\alpha)Y(\beta)n$, where $Z(\cdot)$ and $Y(\cdot)$ represent the rotation matrix around $z$ and $y$ axes, respectively. For more details of this space, we refer the reader to \cite{moon2012field}.

\subsection{Unit Spherical Space: $S^2\times H$}
\label{sec:s2hpara}
Any point $x$ in unit ball can be uniquely parameterized by an augmented spherical coordinate: $(s(\alpha, \beta), h)$, where $s(\alpha, \beta) \in S^2$ denotes its location when projected to unit sphere, and $h\in H = [0, 1]$ represents the radial distance to the origin. It is obvious that this parametrization is bijective.

\subsubsection{Rotation Transformation}
Consider an arbitrary rotation transformation $Q$ applied to a point $x\in S^2\times H$:
\begin{equation}
\label{eq:rots2h}
    \begin{split}
        Qx(\alpha, \beta, h) &= (Qs(\alpha,\beta), h) \\
           &= (QZ(\alpha)Y(\beta)n, h).
    \end{split}
\end{equation}
Intuitively, $Q$ rotates the point around the origin, while keeping its radial distance towards the origin.

\subsection{3D Rotation Group: $SO(3)$}
\label{sec:so3para}
The 3D rotation group, often denoted $SO(3)$, which is termed ``special orthogonal group'', is the group of all rotations about the origin of three-dimensional Euclidean space $\mathbb{R}^3$ under the operation of composition. Almost every element (except for singularities) in $SO(3)$ can be uniquely parameterized by ZYZ-Euler angles~\cite{siciliano2010robotics}: $(\alpha, \beta, \gamma)$, where $\alpha \in [0, 2\pi], \beta \in [0, \pi], $ and $\gamma \in [0, 2\pi]$. In matrix form, for any element $R\in SO(3)$, $R(\alpha, \beta, \gamma) = Z(\alpha)Y(\beta)Z(\gamma)$. This parametrization is also bijective almost everywhere.

\subsubsection{Rotation Transformation}
Consider an arbitrary rotation transformation $Q$ applied to another rotation $R\in SO(3)$, which is a transformation composition:
\begin{equation}
    \begin{split}
        QR(\alpha, \beta, \gamma) &= QZ(\alpha)Y(\beta)Z(\gamma).
    \end{split}
\end{equation}
\section{PRIN: An Exact Point-wise Rotation Invariant Network}
In this section, we discuss the development of our point-wise rotation invariant algorithm. In Section~\ref{sec:input}, we propose a density aware adaptive sampling module to correct the distortion in spherical voxels. In Section~\ref{sec:rotinv}, we propose a special spherical voxel convolution and prove that it is rotation equivariant (e.g., point-wise rotation invariant) theoretically. Besides, we also derive a sufficient condition on convolutional filters to ensure this rotation equivariance.
\subsection{Problem Statement}
Given a set of unordered points $\mathcal{X} = \{x_i\}$ with $x_i \in \mathbb{R}^{d^{in}}$ and $i=1,\cdots,N$, where $N$ denotes the number of input points and $d^{in}$ denotes the dimension of input features at each point, which can be positions, colors, etc. Our goal is to produce a set of point-wise features $\mathcal{Y} = \{y_i\}$ with $y_i \in \mathbb{R}^{d^{out}}$ and $i=1,\cdots,N$, which are invariant to input orientations. PRIN can be modeled as a rotation invariant function $\mathcal{F}:\mathcal{X}\mapsto\mathcal{Y}$. Although the problem and the method developed are general, we focus on the case $d_{in}=3$ using only Euclidean coordinates as the input. To implement the function $\mathcal{F}$, we design mainly three modules: (1) Density Aware Adaptive Sampling (DAAS) module $\Gamma: \mathbb{R}^{N\times 3}\to\mathbb{R}^{S^2\times H}$ that constructs spherical signals; (2) Spherical Voxel Convolution (SVC) module $\Phi: \mathbb{R}^{S^2\times H\times C_{in}}\to\mathbb{R}^{S^2\times H\times C_{out}}$ that extracts rotation invariant features; (3) Point Re-sampling module $\Lambda: \mathbb{R}^{S^2\times H\times C_{out}}\to\mathbb{R}^{N\times C_{out}}$ that re-samples points from spherical signals. We will explain these modules in the following sections and the whole pipeline is shown in Figure~\ref{fig:pipeline}.

\subsection{Density Aware Adaptive Sampling}
\label{sec:input}
In this step, the objective is to build spherical signals from irregular point clouds. Nonetheless, if we sample point clouds uniformly into regular spherical voxels, we will meet a problem: points around the pole appear to be more sparse than those around the equator in spherical coordinates, which brings a bias to the resulting spherical voxel signals. 
To address this problem, we propose Density Aware Adaptive Sampling (DAAS). DAAS leverages a non-uniform filter to adjust the density discrepancy brought by spherical coordinates, thus reducing the bias. This process can be modeled as $\Gamma:\mathbb{R}^{N\times 3}\to\mathbb{R}^{S^2\times H}$.


\subsubsection{Spherical Distortion}
Specifically, we divide unit spherical space $S^2\times H$ into spherical voxels, which are indexed by $(i, j, k) \in I\times J\times K$, where $I\times J\times K$ is the spatial resolution, also known as the bandwidth~\cite{kostelec2007soft}. Each spherical voxel is represented by its center $(s(a_i, b_j), c_k)$ with $a_i = \frac{i}{I}\cdot2\pi$, $b_j = \frac{j}{J}\cdot\pi$, and $c_k= \frac{k}{K}$. The division of unit spherical space is shown in the left of Figure~\ref{fig:density}.

Then we sample the input signal into a tensor of equal-sized measures of Euler angles $(\alpha,\beta,\gamma)$ in $SO(3)$, with the bijective mapping between $SO(3)$ and $S^2\times H$ (Theorem~\ref{thm:bij}).  They are leveraged for discrete spherical voxel convolution discussed later in Section~\ref{sec:rotinv}. This representation is shown in the right of Figure~\ref{fig:density}.

However, such an equal-angle discretization introduces a non-uniform or distorted distribution of the input signal in Euclidean space. This is illustrated in the middle of Figure~\ref{fig:density}. Since the size of the spherical voxel is smaller around the pole than on the equator, the input signal gets distorted around the pole when we stretch non-equal-sized spherical voxels to equal-sized angle bins.

To address this issue and obtain an unbiased estimate of the input signal, Density Aware Adaptive Sampling (DAAS) is proposed. Formally, we map each point $x_n\in \mathcal{X}$ from Euclidean space $\mathbb{R}^3$ to unit spherical space $S^2\times H$ by calculating its spherical coordinates $x_n = (s(\alpha_n,\beta_n),h_n)$,
and then calculate the spherical signal $f : S^2\times H \to\mathbb{R}$ as follows:
\begin{equation}
\label{eq:change}
\begin{split}
    f(a_i, b_j, c_k) = & \frac{\sum\limits^N_{n=1}w_n\cdot(\xi - \|h_n - c_k\|)}{\sum\limits^N_{n=1}w_n},
\end{split}
\end{equation}
where $w_n$ is a normalizing factor that is defined as
\begin{equation}
\label{eq:wt}
\begin{split}
    w_n =\  &\mathbf{1}(\|\alpha_n - a_i\| < \xi) \\ 
            \cdot &\mathbf{1}(\|\beta_n - b_j\| < \eta\xi) \\ 
            \cdot &\mathbf{1}(\|h_n - c_k\| < \xi),
\end{split}
\end{equation}
where $\xi$ is a predefined threshold filter width and $\eta$ is the Density Aware Adaptive Sampling Factor. $f$ can be viewed as an unbiased version of the empirical distribution of $x$, except we use $(\xi - \|h_n - c_k\|)$ instead of Dirac delta because it captures information along the $H$ axis, which is orthogonal to $S^2$ and is invariant under random rotations.

\begin{theorem}[Density Aware Adaptive Sampling Factor]
$\eta = sin(\beta)$ is the Density Aware Adaptive Sampling Factor, accounting for distorted spherical voxels in Euclidean spaces.
\end{theorem}

\begin{proof}
To get the relationship between the differential spherical volumes and rectangular (e.g., Euclidean) volumes, we need to calculate the Jacobian of the transformation from spherical coordinates to Euclidean coordinates. Denote some Euclidean coordinate as $(x, y, z)$ and its corresponding spherical coordinate as $(\alpha, \beta, h)$, we have~\cite{thomas2014early},
\begin{equation}
    \begin{split}
        x &= h\mathrm{sin}(\beta)\mathrm{cos}(\alpha),\\
        y &= h\mathrm{sin}(\beta)\mathrm{sin}(\alpha), \\
        z &= h\mathrm{cos}(\beta).
    \end{split}
\end{equation}
Denote the Jacobian as $J$:
\begin{equation}
J = 
\begin{bmatrix}
    \frac{\partial x}{\partial\alpha}      & \frac{\partial x}{\partial\beta}  & \frac{\partial x}{\partial h} \\
    \frac{\partial y}{\partial\alpha}      & \frac{\partial y}{\partial\beta}  & \frac{\partial y}{\partial h} \\
    \frac{\partial z}{\partial\alpha}      & \frac{\partial z}{\partial\beta}  & \frac{\partial z}{\partial h}
\end{bmatrix} = h^2\mathrm{sin}(\beta),
\end{equation}
This involves $h$ and $\mathrm{sin}(\beta)$. $h$ does reflect the volume difference of spherical voxels when placed at different distances from the ball center. However, this change is isotropic and does not cause the volume distortion between the equator and the pole, which is shown in Figure~\ref{fig:density}. The other factor $\mathrm{sin}(\beta)$ is therefore the only reason for distorted spherical voxels. In other words, if we normalize the spherical voxel by setting $h=1$, we get $J_{|h=1} = sin(\beta)$.
\end{proof}

\begin{figure}[h]
\begin{center}
    \includegraphics[width=\linewidth]{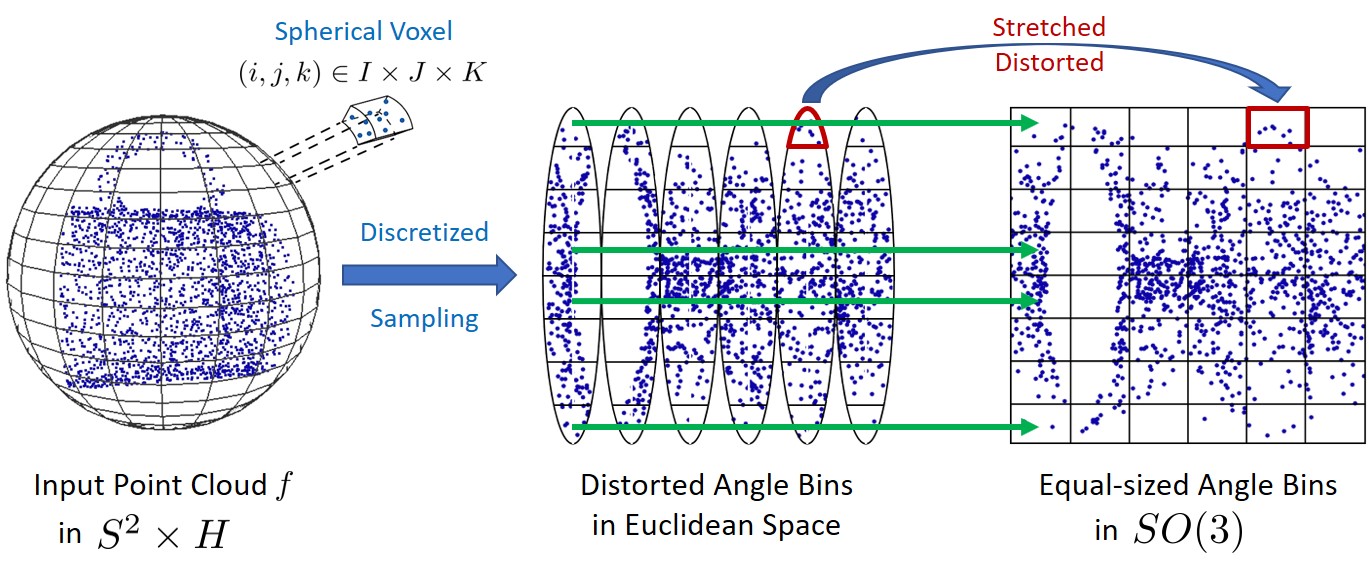}
\end{center}
    \caption{Spherical distortion. An input point cloud is first sampled into discretized angle bins, though such equal-sized angle bins are distorted in Euclidean space. Therefore, we propose Density Aware Adaptive Sampling (DAAS).}
\label{fig:density}
\end{figure}

\subsubsection{Comparison with Other Spherical Representations}
To recap, we create equal-sized angle bins/measures from distorted spherical voxels to form a compact 3D tensor in $SO(3)$. This tensor is later used by discrete spherical voxel convolution to achieve point-wise rotation invariance. 

This is different from what is proposed in SPH3D-GCN~\cite{lei2020spherical}, which also leverages spherical voxel. In SPH3D-GCN, the authors do not consider rotation invariance, and utilize a simple uniform sampling of input signals. This is okay as long as the bias introduced by distorted spherical voxels is consistent in both training and testing. However, if an arbitrary rotation is introduced during testing but no rotation when training, the bias is inevitable, and their method will give an inferior result. This is also confirmed in our ablation studies, referred to as ``uniform sampling''.

Besides, one may consider using HealPix~\cite{gorski2005healpix} that produces a subdivision of a spherical surface in which each pixel covers the same surface area as every other pixel. Though it produces equal sized surface area for each pixel, the center of each pixel is not necessarily aligned with the grid of general discrete Fast Fourier Transforms (FFTs) on the rotation group~\cite{kostelec2008ffts}, which is required for a fast computation of the spherical voxel convolution discussed in Section~\ref{sec:rotinv}. HealPix, though uniformly discretizes voxels in Euclidean space, produces non-uniform discretization in $SO(3)$ space, making the separation of variables impossible. As a result, FFTs are not applicable for HealPix.  For more details about the requirement for FFTs on the rotation group, we refer the reader to Chapter 3 of \cite{kostelec2008ffts}.

\subsection{Spherical Voxel Convolution} 
\label{sec:rotinv}
Given some spherical voxel signal $f: S^2\times H\to\mathbb{R}^{C_{in}}$, we introduce Spherical Voxel Convolution (SVC) to further refine the features, while keeping it rotation invariant. This step implements the operator $\Phi: \mathbb{R}^{S^2\times H \times C_{in}} \to \mathbb{R}^{S^2\times H \times C_{out}}$.

Compared with Spherical CNN~\cite{esteves2018learning} where only spherical signals defined in $S^2$ get convolved. We propose the spherical voxel convolution, which takes spherical signals defined in $S^2\times H$ as the input. 
To understand spherical voxel convolution, we first give some definitions and annotations. Here, we only consider functions with $C_{in}=C_{out}=1$ to reduce clutter, while extensions to higher-dimension functions are straight-forward.

In order to define spherical voxel convolution, we need to first bridge $S^2\times H$ and $SO(3)$ by several theorems and lemmas.
\begin{theorem}[Bijection from $S^2\times H$ to $SO(3)$]
\label{thm:bij}
There exists a bijective map (almost everywhere) $\mathcal{T}: S^2\times H \to SO(3)$, such that $\mathcal{T}(s(\alpha, \beta), h) = R(\alpha, \beta, 2\pi h) = Z(\alpha)Y(\beta)Z(2\pi h)$.
\end{theorem}

\begin{proof}
This map is onto because almost every element $R\in SO(3)$ can be parameterized by ZYZ-Euler angles~\cite{siciliano2010robotics}: $(\alpha,\beta,\gamma)$, and there is some $x = (s(\alpha, \beta), \frac{\gamma}{2\pi})$ such that $\mathcal{T}(x) = R(\alpha, \beta, \gamma)$. To prove that this map in injective, suppose $x_1 \neq x_2$, by the bijective parameterization of $S^2\times H$, $(\alpha_1, \beta_1, h_1) \neq (\alpha_2, \beta_2, h_2)$, therefore $\mathcal{T}(x_1) \neq \mathcal{T}(x_2)$ by the bijective parameterization of $SO(3)$.
\end{proof}

\begin{lemma}
\label{lma:tx}
For any $x \in S^2\times H$ and $Q\in SO(3)$, there exists some $\theta$, such that:
\begin{align}
    \mathcal{T}(Qx) = Q\mathcal{T}(x)Z(\theta),
\end{align}
where $Z(\theta)$ is the rotation around $z$ axes by $\theta$.
\end{lemma}
\begin{proof}
Writing left-hand side out and leverage Equation~\ref{eq:rots2h}, we have:
\begin{align}
    \mathcal{T}(Qx(\alpha,\beta, h)) &= \mathcal{T}(Qs(\alpha,\beta),h) \\
    &= \mathcal{T}(QZ(\alpha)Y(\beta)n,h), \label{eq:tq}
\end{align}
Notice that $QZ(\alpha)Y(\beta)$ is a general rotation and can be uniquely parameterized as $Z(\alpha')Y(\beta')Z(\gamma')$. Substitute this into Equation~\ref{eq:tq}:
\begin{align}
    \mathcal{T}(Qx) &= \mathcal{T}(Z(\alpha')Y(\beta')Z(\gamma')n,h)\\
    &= \mathcal{T}(Z(\alpha')Y(\beta')n,h) \label{eq:l2} \\
    &= Z(\alpha')Y(\beta')Z(2\pi h) \\
    &= QZ(\alpha)Y(\beta)Z(-\gamma')Z(2\pi h) \label{eq:l4}\\
    &= QZ(\alpha)Y(\beta)Z(2\pi h)Z(-\gamma') \label{eq:l5}\\
    &= Q\mathcal{T}(x)Z(-\gamma').
\end{align}
Equation~(\ref{eq:l2}) follows from that rotations around $z$ axes fix the North Pole; Equation~(\ref{eq:l4}) follows because $QZ(\alpha)Y(\beta)=Z(\alpha')Y(\beta')Z(\gamma')$; Equation~(\ref{eq:l5}) follows since rotations around $z$ axes are commutative.
\end{proof}

\begin{definition}[Adjoint Function]
We define the adjoint function of $f$ as follows:
\begin{align}
    f_\mathcal{T}(x) = f(\mathcal{T}^{-1}(x)).
\end{align}
\end{definition}
\begin{remark}
Notice that $f_\mathcal{T}$ is also the pull-back function~\cite{solomon2017computational} of the measurable push-forward map $\mathcal{T}$ (up to a constant):
\begin{align}
    \int_{SO(3)}f_\mathcal{T}(x)dx &= \int_{SO(3)}f(\mathcal{T}^{-1}(x))dx \\
    &= \int_{S^2\times H}f(y)d(\mathcal{T}(y)) \\
    &= 2\pi\int_{S^2\times H}f(y)dy \label{eq:pushback},
\end{align}
where we make use of $y:=\mathcal{T}^{-1}(x)$ and Equation~\ref{eq:pushback} follows by Theorem~\ref{thm:bij}.
\end{remark}

Now, we are ready to give a formal definition of spherical voxel convolution, by leveraging $SO(3)$ group convolution.
\begin{definition}[Spherical Voxel Convolution]
\label{thm:svc}
Spherical voxel convolution, evaluated at $p\in S^2\times H$, is defined as:
\begin{equation}
    \label{eq:voxelconv}
    \begin{split}
        [\psi\star f] (p) &= \int_\gamma[\psi_\mathcal{T}\star f_\mathcal{T}](\mathcal{T}(p)Z(\gamma))d\gamma \\
        &= \int_\gamma\int_{SO(3)}\psi_\mathcal{T}(R^{-1}\mathcal{T}(p))f_\mathcal{T}(RZ(\gamma))dRd\gamma.
    \end{split}
\end{equation}
$\psi, f: S^2\times H\to\mathbb{R}$, where $f$ is the input signal and $\psi$ is the filter. Intuitively, our spherical voxel convolution is the corresponding adjoint $SO(3)$ convolution averaged over the coset defined by $Z(\gamma)$.
\end{definition}

\begin{definition}[Rotation Operator~\cite{cohen2018spherical}]
We introduce a rotation function operator $L_Q$: 
\begin{align}
    [L_Qf](x) = f(Q^{-1}x), Q\in SO(3).
\end{align}
It follows that:
\begin{align}
    [L_Qf]_\mathcal{T}(x) = f_\mathcal{T}(\mathcal{T}(Q^{-1}x)) = f_\mathcal{T}(Q^{-1}\mathcal{T}(x)Z(\theta)).
\end{align}
\end{definition}

\begin{theorem}[\textbf{Main Result: Rotation Invariance}]
\label{thm:ri}
The spherical voxel convolution is point-wise rotation invariant: $[\psi\star L_Qf] (Qp) = [\psi\star f] (p)$, when $\psi_\mathcal{T}$ is constant on the right latitude: $\psi_\mathcal{T}(R)\equiv\psi_\mathcal{T}(RZ(\theta))$ for any $R,\theta$.
\end{theorem}

\begin{proof}

Suppose that the input point cloud is rotated by an arbitrary rotation $Q$, we have $p\to Qp$. As a consequence, the corresponding spherical signal is also rotated: $f\to L_Qf$,  since $f$ is sampled from the original point cloud. We are now ready to prove the rotation invariance by applying SVC to the rotated input:
\begin{align}
    &\hspace{0.5em}[\psi\star L_Qf] (Qp)\\
        =& \int_\gamma\int_{SO(3)}\psi_\mathcal{T}(R^{-1}\mathcal{T}(Qp))f_\mathcal{T}(Q^{-1}RZ(\theta+\gamma))dRd\gamma\\
        =& \int_\gamma\int_{SO(3)}\psi_\mathcal{T}(R^{-1}Q\mathcal{T}(p))f_\mathcal{T}(Q^{-1}RZ(\theta+\gamma))dRd\gamma \label{eq:l6}\\
        =& \int_{SO(3)}\psi_\mathcal{T}(R^{-1}\mathcal{T}(p))\int_\gamma f_\mathcal{T}(RZ(\theta+\gamma))d\gamma dR \label{eq:l7}\\
        =& \int_{SO(3)}\psi_\mathcal{T}(R^{-1}\mathcal{T}(p))\int_\gamma f_\mathcal{T}(RZ(\gamma))d\gamma dR \label{eq:l8}\\
        =& \int_\gamma\int_{SO(3)}\psi_\mathcal{T}(R^{-1}\mathcal{T}(p))f_\mathcal{T}(RZ(\gamma))dRd\gamma \\
        =&\hspace{0.5em} [\psi\star f] (p).
\end{align}
Equation~(\ref{eq:l6}) follows from Lemma~\ref{lma:tx} and $\psi_\mathcal{T}(R)\equiv\psi_\mathcal{T}(RZ(\theta))$; Equation~(\ref{eq:l7}) follows from the right invariance of $SO(3)$ group integral~\cite{harmso3}; Equation~(\ref{eq:l8}) is a result of corollary (2.1) in~\cite{harmso3}.
\end{proof}

It is obvious that SVC extracts the same feature no matter how the input point cloud rotates, which ensures point-wise \textbf{rotation invariance}.

In practice, Spherical Voxel Convolution (SVC) can be efficiently computed by Fast Fourier Transform (FFT)\cite{kostelec2008ffts}. Convolutions are implemented by first doing FFT to convert both the input and filters into spectral domain, then multiplying them and converting the results back to spatial domain, using Inverse Fast Fourier Transform (IFFT)\cite{kostelec2008ffts}.

\subsubsection{Finding Filter $\psi$ to Achieve Rotation Invariance}
Notice that to achieve the rotation invariance stated in Theorem~\ref{thm:ri}, we need a filter such that $\psi_\mathcal{T}(R)\equiv\psi_\mathcal{T}(RZ(\theta))$, for any $R,\theta$. Leveraging the Euler angle representation of $R$, we can reform the condition as:
\begin{align}
    \psi_\mathcal{T}(Z(\alpha)Y(\beta)Z(\gamma))\equiv\psi_\mathcal{T}(Z(\alpha)Y(\beta)Z(\gamma + \theta)), \forall \alpha,\beta,\gamma,\theta.
\end{align}
In other words,
\begin{align}
\label{eq:psi}
    \psi_\mathcal{T}(Z(\alpha)Y(\beta)Z(\gamma))\equiv\psi_\mathcal{T}(Z(\alpha)Y(\beta)Z(\gamma')), \forall \alpha,\beta,\gamma,\gamma'.
\end{align}
During implementation, we regard $\psi_\mathcal{T}$ as a 3D tensor with dimensions corresponding to the three Euler angles, we have that $\psi_\mathcal{T}$ is constant on the third dimension.

To qualitatively illustrate this, we plot a sample signal $f$ and its rotated version $L_Q{f}$, together with a filter $\psi$ that satisfy the constraint~\ref{eq:psi} in Figure~\ref{fig:equiv}. We see that after our point-wise rotation invariant spherical voxel convolution, the output is rotation equivariant; and for each individual point, exact rotation invariance is achieved with an error up to the numeric precision.

\begin{figure*}[ht]
    \centering
    \includegraphics[width=\linewidth]{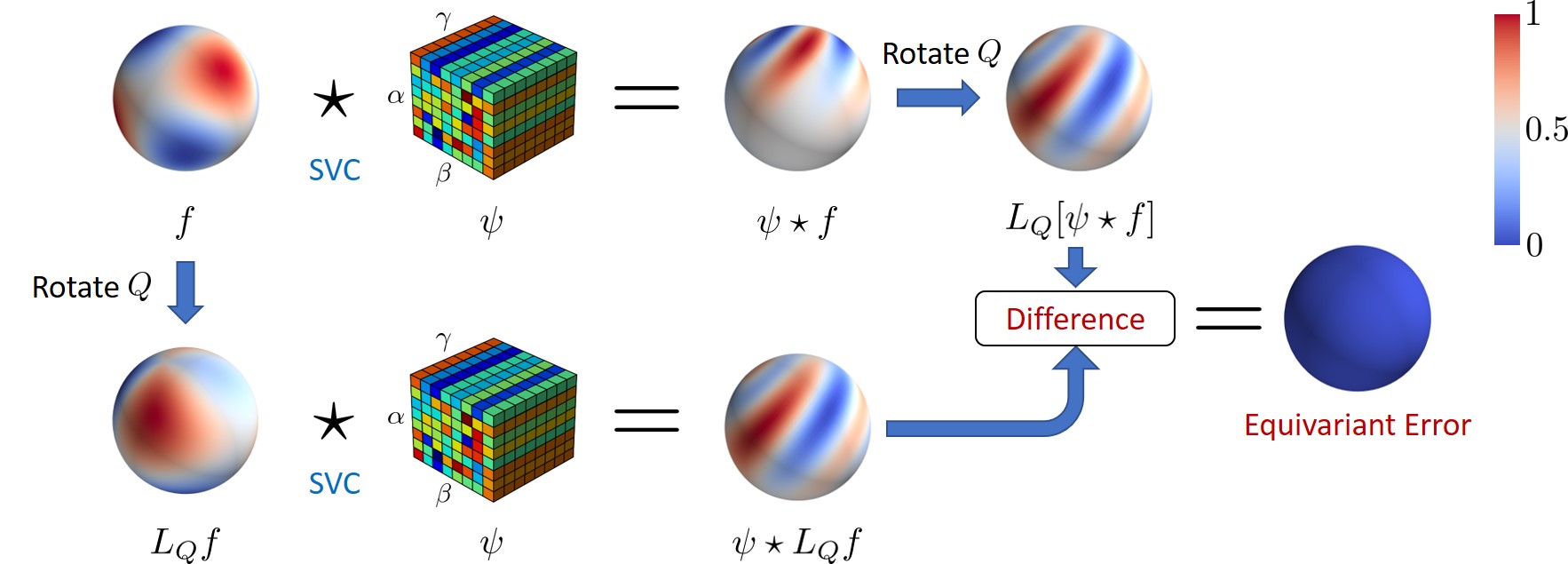}
    \caption{Our spherical voxel convolution that achieves point-wise rotation invariance, a.k.a. rotation equivariant. Applying a rotation to the input signal and do Spherical Voxel Convolution is equivalent to applying the same rotation to the original convolution results. The absolute equivariant error is almost zero, with maximum $4.77\times 10^{-7}$ due to the floating point precision. Also notice that $\psi$ is constant on the third dimension $\gamma$ in order to fulfill the constraint~\ref{eq:psi}. We slightly abuse the notation $\psi$ for $\psi_\mathcal{T}$ for a better illustration.}
    \label{fig:equiv}
\end{figure*}

\subsection{Point Re-sampling}
\label{sec:output}
After Spherical Voxel Convolution (SVC), we re-sample features  at the location of original points, with \textit{Trilinear Interpolation} as our operator $\Lambda : \mathbb{R}^{S^2\times H \times C}\to\mathbb{R}^{N \times C}$. Each point's feature is a weighted average of the nearest eight voxels, where the weights are inversely related to the distances to these spherical voxels.
Formally, denote the target sampled point $x$ as the spherical coordinate $(s(\alpha,\beta),h)$; the discrete feature map before point re-sampling as $\bar{F} \in \mathbb{R}^{I\times J \times K \times C}$, where $I \times J \times K$ are predefined spatial resolution for $S^2\times H$; the point-wise feature map after re-sampling as $F \in \mathbb{R}^{N \times C}$, we have:

\begin{align}
    &F(x, \cdot) = \\
    &\sum_{a,b,c\in\{0,1\}} \frac{w_{abc}\cdot \bar{F}(\left\lfloor{\frac{I}{2\pi}\alpha}\right\rfloor + a, \left\lfloor{\frac{J}{\pi}\beta}\right\rfloor + b, \left\lfloor{K h}\right\rfloor + c, \cdot)}{\sum_{a',b',c'\in\{0,1\}} w_{a'b'c'}},
\end{align}
where
\begin{align}
    w_{abc} =\ &(1-a + (2a-1)(\frac{I}{2\pi}\alpha - \left\lfloor{\frac{I}{2\pi}\alpha}\right\rfloor)) \\
    \cdot &(1-b + (2b-1)(\frac{J}{\pi}\beta - \left\lfloor{\frac{J}{\pi}\beta}\right\rfloor)) \\
    \cdot &(1-c + (2c-1)(K h - \left\lfloor{K h}\right\rfloor)),
\end{align}
and $\left\lfloor{\cdot}\right\rfloor$ is the integer floor function.

Finally, $F$ are passed through several fully connected layers to get refined point-wise features.  

\subsection{Architecture}
To summarize, our rotation invariant function $\mathcal{F}$ is a concatenation of the three modules/operators mentioned above: $\mathcal{Y} = \Lambda(\Phi(\Gamma(\mathcal{X})))$. We first transform points from Euclidean space to unit spherical space by operator $\Gamma$, and then conduct a rotation invariant feature transform by operator $\Phi$, and finally re-sample the point features in Euclidean space by operator $\Lambda$.

After extracting point-wise rotation invariant features, we are able to do part segmentation by concatenating some fully connected layers. Our network could also realize object classification by placing a different head. In this case, we maxpool all the features in spherical voxels and pass this global feature through several fully connected layers, as shown in Figure~\ref{fig:pipeline}. 

\begin{table*}[t]
\begin{center}
\resizebox{\textwidth}{!}{
\begin{tabular}{l|c|c|cccccccccccccccc|c|c}
\toprule[1pt]
\multirow{2}*{}  & \multicolumn{18}{c|}{Arbitrary Rotation} & \multicolumn{2}{c}{No Rotation}\\
\cline{2-21}
~ & \tabincell{c}{avg.\\ inst.} & \tabincell{c}{avg.\\ cls.}  & \tabincell{c}{air\\plane} & bag & cap & car & chair & \tabincell{c}{ear\\ phone} & guitar & knife & lamp & laptop & \tabincell{c}{motor\\ bike} & mug & pistol & rocket & \tabincell{c}{skate\\ board} & table & \tabincell{c}{avg.\\ inst.} & \tabincell{c}{avg.\\ cls.} \\
\hline
PointNet~\cite{8099499} & 31.30 & 29.38 & 19.90 & 46.25 & 43.27	& 20.81	& 27.04 & 15.63 & 34.72	& 34.64	& 42.10	& 36.40	& 19.25	& 49.88	& 33.30	& 22.07	& 25.71 & 29.74 & 83.15 & 78.95\\
PointNet++~\cite{qi2017pointnet++} & 36.66 & 35.00 & 21.90 & 51.70 & 40.06 & 23.13 & 43.03 & 9.65 & 38.51 & 40.91 & 45.56 & 41.75 & 18.18 & 53.42 & 42.19 & 28.51 & 38.92 & 36.57 & 84.63 & 81.52 \\
RS-Net~\cite{huang2018recurrent} & 50.38 & 32.99 & 38.29 & 15.45 & 53.78 & 33.49 & 60.83 & 31.27 & 9.50 & 43.48 & 57.37 & 9.86 & 20.37 & 25.74 & 20.63 & 11.51 & 30.14 & 66.11 & 84.92 & 81.41\\
PCNN~\cite{Atzmon:2018:PCN:3197517.3201301} & 28.80 & 31.72  & 23.46 & 46.55 & 35.25 & 22.62 & 24.27 & 16.67 & 32.89 & 39.80 & 52.18 & 38.60 & 18.54 & 48.90 & 27.83 & 27.46 & 27.60 & 24.88  & 85.13 & 81.80\\
SPLATNet~\cite{su18splatnet} & 32.21 & 38.25 & 34.58 & 68.10 & 46.96 & 19.36 & 16.25 & 24.72 & 88.39 & 52.99 & 49.21 & 31.83 & 17.06 & 48.56 & 21.20 & 34.98 & 28.99 & 28.86 & 84.97 & 82.34 \\
DGCNN~\cite{dgcnn} & 43.79 & 30.87 & 24.84 & 51.29 & 36.69 & 20.33 & 30.07 & 27.86 & 38.00 & 45.50 & 42.29 & 34.84 & 20.51 & 48.74 & 26.25 & 26.88 & 26.95 & 28.85 & 85.15 & 82.33 \\
SO-Net~\cite{li2018so} & 26.21 & 14.37 & 21.08 & 8.46 & 1.87 & 11.78 & 27.81 & 11.99 & 8.34 & 15.01 & 43.98 & 1.81 & 7.05 & 8.78 & 4.41 & 6.38 & 16.10 & 34.98 & 84.83 & 81.16 \\
SpiderCNN~\cite{xu2018spidercnn} & 31.81 & 35.46 & 22.28 & 53.07 & 54.2 & 22.57 & 28.86 & 23.17 & 35.85 & 42.72 & 44.09 & 55.44 & 19.23 & 48.93 & 28.65 & 25.61 & 31.36 & 31.32 & \textbf{85.33} & \textbf{82.40} \\
\hline 
SHOT+PointNet~\cite{shot} & 32.88 & 31.46 & 37.42 & 47.30 & 49.53 & 27.71 & 28.09 & 16.34 & 9.79 & 27.66 & 37.33 & 25.22 & 16.31 & 50.91 & 25.07 & 21.29 & 43.10 & 40.27 & 32.75 & 31.25 \\
CGF+PointNet~\cite{khoury2017learning} & 50.13 & 46.26 & 50.97 & 70.34 & 60.44 & 25.51 & 59.08 & 33.29 & 50.92 & 71.64 & 40.77 & 31.91 & 23.93 & 63.17 & 27.73 & 30.99 & 47.25 & 52.06 & 50.13 & 46.31 \\
\hline
SRINet~\cite{sun2019srinet} & 76.95 & - & - & - & - & - & - & - & - & - & - & - & - & - & - & - & - & - & 76.95 & - \\
RIConv~\cite{zhang2019rotation} & 79.31 & 74.60 & 78.64 & 78.70 & 73.19 & 68.03 & 86.82 & 71.87 & 89.36 & 82.95 & 74.70 & 76.42 & 56.58 & 88.44 & 72.16 & 51.63 & 66.65 & 77.47 & 79.55 & 74.43 \\
Kim et al.~\cite{kim2020rotation} & 79.56 & 74.41 & 77.53 & 73.43 & 76.95 & 66.13 & 87.22 & \textbf{75.44} & 87.42 & 80.71 & 78.44 & 71.21 & 51.09 & 90.76 & 73.69 & 53.86 & 68.10 & 78.62 & 79.92 & 74.69 \\
Li et al.~\cite{li2021rotation} & 82.17 & 78.78 & 81.49 & 80.07 & \textbf{85.55} & 74.83 & \textbf{88.62} & 71.34 & 90.38 & 82.82 & \textbf{80.34} & 81.64 & 68.87 & 92.23 & 74.51 & 54.08 & 74.59 & 79.11 & 82.47 & 79.40 \\
\hline
\textbf{Ours(PRIN)} & 71.20 & 66.75 & 69.29 & 55.90 & 71.49 & 56.31 & 78.44 & 65.92 & 86.01 & 73.58 & 66.97 & 59.29 & 47.56 & 81.47 & 71.99 & 49.02 & 64.70 & 70.12 & 72.04 & 68.39 \\
\textbf{Ours(SPRIN)} & \textbf{82.67} & \textbf{79.50} & \textbf{82.07} & \textbf{82.01} & 76.48 & \textbf{75.53} & 88.17 & 71.45 & \textbf{90.51} & \textbf{83.95} & 79.22 & \textbf{83.83} & \textbf{72.59} & \textbf{93.24} & \textbf{78.99} & \textbf{58.85} & \textbf{74.77} & \textbf{80.31} & 82.59 & 79.31 \\
\bottomrule[1pt]
\end{tabular}}
\end{center}   
\caption{\textbf{Shape part segmentation results on ShapeNet part dataset.} Both average instance and average class IoUs (\%) are reported. All models are trained on the non-rotated  training set, then evaluated on the non-rotated and rotated test set, indicated by the labels.}
\label{tab:compare_miou}
\end{table*}

\section{SPRIN: Sparse Point-wise Rotation Invariant Network}
Though PRIN achieves point-wise rotation invariance by conducting DAAS and SVC, it is limited by the resolution of spherical voxel grids in $S^2\times H$. In practice, the maximum resolution allowed by GPU memory is only $64^3$, while a large quantity of spherical voxels are empty containing no points. This is extremely inefficient and motivates us to propose a sparse version of PRIN, by directly taking the point cloud without sampling into dense grids.

In this section, we provide a sparse version of PRIN, which directly operates on the original sparse point cloud, making it more efficient and achieve the state-of-the-art performance. We borrow the idea of spherical voxel convolution and propose sparse correlation on the input points. Besides, since our sparse correlation directly operates on the input points, there is no need to conduct DAAS or point re-sampling to convert input or output signals across different domains. Details are given below.

\subsection{Sparse Rotation Invariance}
\begin{definition}[Point Cloud]
A point cloud $f$ can be viewed as a proper (normalized) empirical discrete probabilistic distribution of input points:
\begin{align}
    f(x) = \frac{1}{N}\sum_{i}^N\delta(x-x_i),
\end{align}
where $\delta$ is the Dirac delta function, $x_i$ are input points. This is different from that in Section~\ref{sec:input}, where we reconstruct the dense signal through an adaptive sampling method. 
\end{definition}

\begin{definition}[Sparse Correlation]
\label{eq:scorr}
In SPRIN, both input and output are the sparse points of the original point cloud. We directly take $S^2\times H$ correlation between the filter and the point cloud, evaluated at some point $x_j$:
\begin{align}
    [\psi\ast f](x_j) &= \int_{S^2\times H} \psi(\mathcal{T}(x_j)^{-1}x)f(x)dx \\
                     &= \frac{1}{N}\sum_i^N\int_{S^2\times H} \psi(\mathcal{T}(x_j)^{-1}x)\delta(x-x_i)dx\\
                     &= \frac{1}{N}\sum_i^N \psi(\mathcal{T}(x_j)^{-1}x_i) \label{eq:sprin}
\end{align}
\end{definition}
\begin{remark}
The correlation output $[\psi\ast f](x_j)$ can be also viewed as the empirical filter expectation $\mathbb{E}_{x\sim f(x)}[\psi(\mathcal{T}(x_j)^{-1}x)]$.
\end{remark}

\begin{theorem}[Sparse Point-wise Rotation Invariance]
The sparse correlation is point-wise rotation invariant when $\psi$ is constant on the left coset of $Z(\theta)$: $\psi(x)\equiv\psi(Z(\theta)x)$, for any $x,\theta$.
\end{theorem}
\begin{proof}
\begin{align}
    &\hspace{0.5em}[\psi\ast L_Qf](Qx_j) \\
    =& \int_{S^2\times H} \psi(\mathcal{T}(Qx_j)^{-1}x)f(Q^{-1}x)dx \\
    =& \int_{S^2\times H} \psi(Z(-\theta)\mathcal{T}(x_j)^{-1}Q^{-1}x)f(Q^{-1}x)dx \\
    =& \int_{S^2\times H} \psi(\mathcal{T}(x_j)^{-1}x)f(x)dx \\ 
    =&\hspace{0.5em}[\psi\ast f](x_j).
\end{align}
\end{proof}

In practice, we sum $\psi$ over $k$ nearest neighbors of $x_j$ to reduce memory footprint.  $\psi$ is implemented as fully connected layers, which is constant on the latitude. Since $[\psi\ast f](x_j)$ can be interpreted as the expectation $\mathbb{E}_{x\sim f(x)}[\psi(\mathcal{T}(x_j)^{-1}x)]$, we randomly select a subset of $k$ nearest neighbors for filtering without bias. This trick is also mentioned in \cite{kim2020rotation} as dilated $k$NN module.

In order to include more context, besides $\mathcal{T}(x_j)^{-1}x_i$, we augment $\psi$ with six rotation invariant features, which are sides and inner-angles of the triangle formed by $x_i$, $x_j$ and $\frac{1}{N}\sum_i^Nx_i$. Though some other rotation invariant features are also available, such as moments $\frac{1}{N}\sum_i^Nx_i^2$, $\frac{1}{N}\sum_i^Nx_i^3$, we empirically find that the triangle formed by $x_i$, $x_j$ and $\frac{1}{N}\sum_i^Nx_i$ gives a better result by the ablation study. We suspect this is because the triangle sides and angles provide more structural information than simple moments.

\subsection{Architecture}
SPRIN architecture is shown in Figure~\ref{fig:sprin}. It progressively conducts sparse rotation invariant spherical correlation~\ref{eq:scorr} between every point and its $k$ nearest neighbors. The dilated $k$NN module is leveraged, which is similar to that of Kim \etal\cite{kim2020rotation}. Given a center point $x_j$ and a dilation rate $d$, the output of the dilated $k$NN search is an index set of $k/d$ points, which are randomly chosen from the nearest $k$ neighbors. To hierarchical propagate low-level information to a larger high-level region, we leverage a similar operation \textit{set abstraction} with PointNet++\cite{NIPS2017_7095}. Our \textit{set abstraction} layer only contains the farthest point sampling module, which defines the evaluating point of following sparse spherical correlation layers. 

For part segmentation, we incorporate two feature propagation layers are concatenated to up-sample the point cloud. This operation is similar to that in PointNet++\cite{NIPS2017_7095} while our model does not induce any weight interpolation. It directly conducts sparse spherical correlation at those up-sampled points.

\begin{figure*}[ht]
\begin{center}
\includegraphics[width=\linewidth]{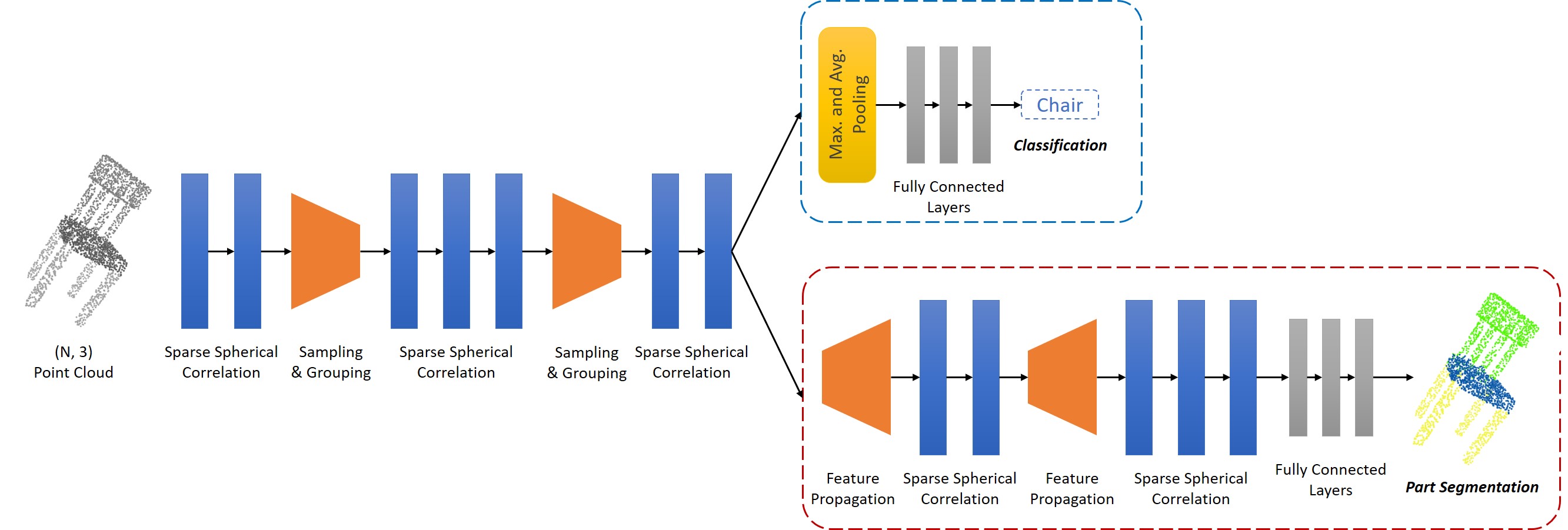}
\end{center}
\caption{\textbf{SPRIN Architecture.}  
SPRIN directly operates on sparse point clouds with several sparse spherical correlation layers. Farthest point sampling and kNN grouping are leveraged to aggregate information from low level to high level. Max and average pooling layers follows to extract global features, and finally a classification score is given by fully-connected layers. 
}
\label{fig:sprin}
\end{figure*}

\section{Experiments}
\label{seq:experiments}

In this section, we arrange comprehensive experiments to evaluate PRIN for point cloud analysis on different tasks. First, we demonstrate that our model can be applied to 3D shape part segmentation and classification with random rotations. Then, we provide some applications on 3D point matching and shape alignment. At last, we conduct an ablation study to validate the design of each part.

\subsection{Implementation Details.} PRIN and SPRIN are all implemented in Python using PyTorch on one NVIDIA GTX 1080 Ti. 
Adam~\cite{kingma2014adam} optimization algorithm with learning rate 1e-2 is employed for PRIN, with a batch size of 8. Adam optimization algorithm with learning rate 1e-3 is employed for SPRIN, with a batch size of 16. The learning rate for PRIN decays with a rate of 0.5 every 5 epochs. 

For PRIN, we set $\xi=1/32$. One spherical voxel convolution layer is utilized with 40 channels. It is followed by two $SO(3)$ group convolution layers with channel 40 and 50. All layers share the same bandwidth, which is 32. Two fully-connected layers with channels 50, 50 are concatenated for part segmentation. 

For SPRIN, the first two sparse spherical correlation layers are evaluated on all input points, with 2-strided 64-NN neighborhoods. Next three sparse spherical correlation layers are evaluated on 128 sampled points from the sampling \& grouping module, with 3-strided 72-NN, 1-strided 32-NN and 1-strided 32-NN neighborhoods, respectively. Then three sparse spherical correlation layers are evaluated on 32 sampled points from the sampling \& grouping module, with 1-strided 32-NN neighborhoods. For classification, the max and avg pooling layer concatenates max-pooling and avg-pooling results, followed by three  fully-connected layers, with 256, 64, 40 channels respectively. For part segmentation, two sparse spherical correlation layers with 1-strided 16-NN and 1-strided 32-NN are proposed at the first up-sampling level. Three sparse spherical correlation layers with 1-strided 32-NN, 2-strided 48-NN and 3-strided 96-NN are proposed at the second up-sampling level. In the end, three fully connected layers with channels 128, 256, 50 are concatenated.

\subsection{Part Segmentation on Rotated Shapes}
\label{seq:part}

\subsubsection{Dataset} ShapeNet part dataset~\cite{Yi16} contains 16,881 shapes from 16 categories in which each shape is annotated with expert verified part labels from 50 different labels in total. Most shapes are composed of two to five parts. We follow the data split in PointNet~\cite{8099499}. 2048 points are sampled per shape for both PRIN and SPRIN.

\subsubsection{Evaluation Results}

Part segmentation is a challenging task for rotated shape analysis. We compare our method with various baseline methods and train each model on the non-rotated training set and evaluate them on the non-rotated and rotated test set. All the results of baselines are either taken from their original papers or, if not available, reproduced with their released code. RIConv~\cite{zhang2019rotation}, Kim \etal~\cite{kim2020rotation} and Li \etal~\cite{li2021rotation} reported their results on arbitrary rotated test shapes with z-axis data augmentation at training time. In order to make a fair comparison, we re-trained and re-evaluated their models with no rotation augmentation at training time.

Qualitative results are shown in Figure~\ref{fig:main}. Both PRIN and SPRIN can accomplish part segmentation task on randomly rotated shapes without seeing these rotations. Even though state-of-the-art vanilla deep learning networks like PointNet++~\cite{qi2017pointnet++}, DGCNN~\cite{dgcnn} and SpiderCNN~\cite{xu2018spidercnn} can achieve fairly good results, they fail on the rotated point clouds. Besides, traditional local features based on local reference frames (LRF) like CGF~\cite{khoury2017learning} and SHOT~\cite{shot} concatenated are also compared. For these descriptors, we train a PointNet-like network with CGF/SHOT features as the input. Although these descriptors are rotation invariant, their performance is inferior to ours. Point-level rotation invariant convolutions such as RIConv~\cite{zhang2019rotation} and Kim \etal~\cite{kim2020rotation} are trained with $z$-axes rotation augmentation but still inferior to our SPRIN network.






Table~\ref{tab:compare_miou} shows quantitative results of PRIN/SPRIN and baseline methods. We compare both mean instance IoU and mean class IoU for all the methods. Our SPRIN network even approaches the performance of modern deep learning methods when tested with no rotation.
Figure~\ref{fig:main} gives some visualization of results between state-of-the-art deep learning methods and PRIN/SPRIN over the ShapeNet part dataset. Biased by the canonical orientation of point clouds in the training set, networks like PointNet just learn a simple partition of Euclidean space, regardless of how the object is oriented. RIConv gives inaccurate segmentation results on part boundaries, while our methods are able to give accurate results under arbitrary orientations.

It should be also noticed that both PRIN and SPRIN have little performance drop when evaluated on the rotated test set, compared with that on the non-rotated set. This is also consistent with our theoretical analysis.



\begin{figure*}[ht]
\begin{center}
\includegraphics[width=\linewidth]{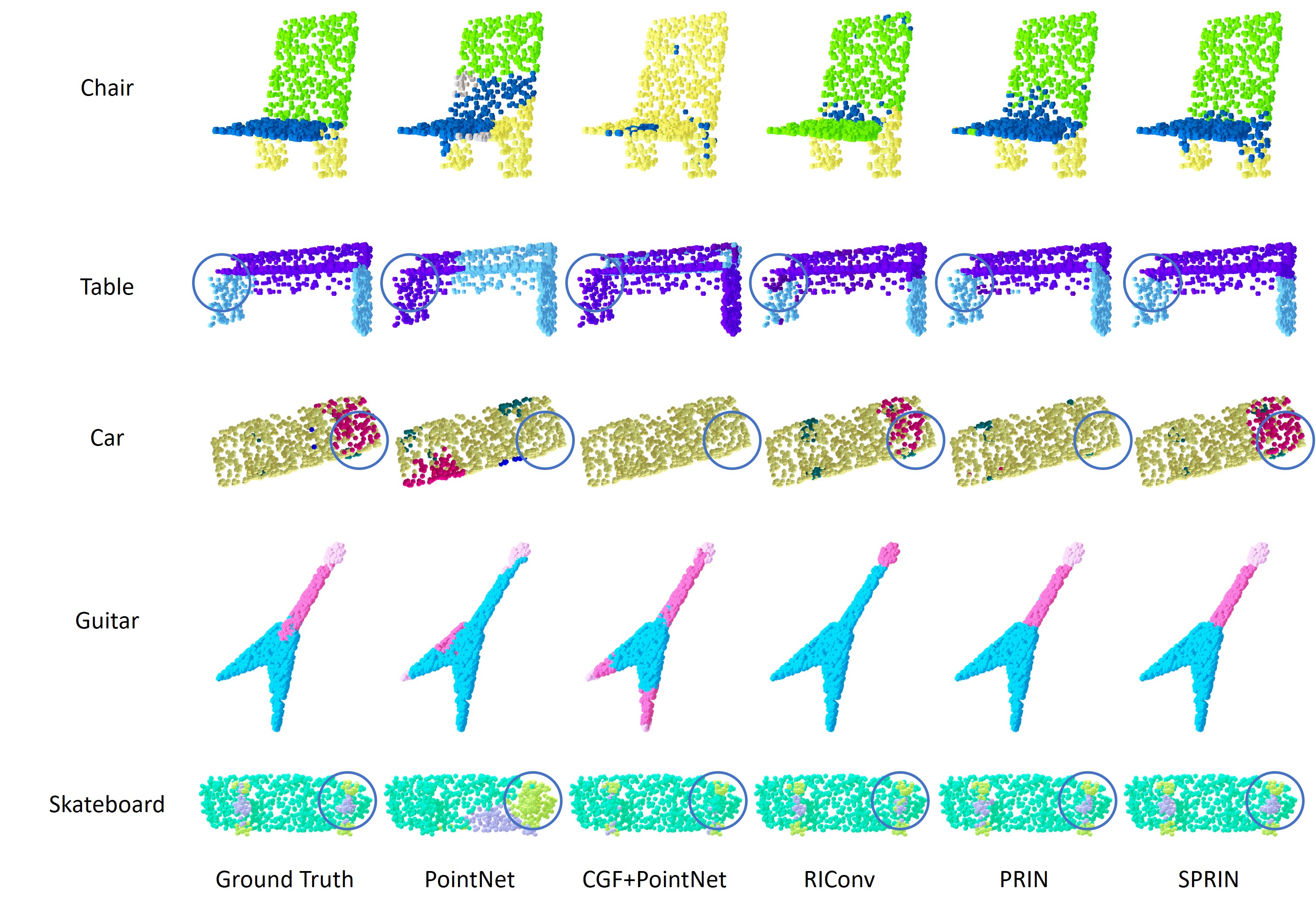}
\end{center}
\caption{\textbf{Visualization of results.} We compare the results of various methods on rotated point clouds, which are trained on the non-rotated dataset. Both PRIN and SPRIN generalize well to unseen orientations.}
\label{fig:main}
\end{figure*}


\subsection{Classification on Rotated Shapes}
In this section, we evaluate our method on ModelNet40 and ScanObjectNN, which are synthetic and real-world datasets respectively.
\subsubsection{Dataset} ModelNet40~\cite{wu20153d} is a 3D shape classification dataset which contains 12,308 shapes from 40 categories. Here, we use its corresponding point clouds dataset provided by PointNet\cite{8099499}.  ScanObjectNN~\cite{uy2019revisiting} is a real-world point cloud object dataset, where each object is retrieved from real scans with potential clutter backgrounds. 2048 points are sampled per shape for both PRIN and SPRIN.

\begin{table}[h]
\begin{center}
\begin{tabularx}{\linewidth}{p{2.5cm}|p{2.4cm}<{\centering}|p{1.6cm}<{\centering}|p{1.6cm}<{\centering}}
\toprule[1pt]
Method & Input & NR & AR \\
\hline
PointNet~\cite{8099499} & \multirow{12}*{\tabincell{c}{Point XYZ}} & 88.45 & 12.47 \\
PointNet++~\cite{qi2017pointnet++} & ~ & 89.82 & 21.35 \\
PCNN~\cite{Atzmon:2018:PCN:3197517.3201301} & ~ & 92.30 & 17.38\\
Point2Sequence~\cite{liu2019point2sequence} & ~ & 92.60 & 10.53 \\
Kd-Network~\cite{8237361} & ~ & 86.20 & 8.49 \\
Spherical CNN~\cite{cohen2018spherical}  & ~ & 81.73 & 55.62\\
DeepSets~\cite{zaheer2017deep} & ~ & 88.73 & 9.72 \\
LDGCNN~\cite{zhang2019linked} & ~ & 92.91 & 17.82\\
SO-Net~\cite{li2018so} & ~ & \textbf{93.44} & 9.64\\
Thomas \etal~\cite{thomas2018tensor} & ~ & 31.24 & 32.29 \\
QE-Net~\cite{zhao2020quaternion} & ~ & 74.43 & 74.07 \\
SPHNet~\cite{poulenard2019effective} & ~ & 87.70 & 86.60 \\
SRINet~\cite{sun2019srinet} & ~ & 87.01 & 87.01 \\ 
RIConv~\cite{zhang2019rotation} & ~ & 87.56 & 87.24 \\
Kim \etal~\cite{kim2020rotation} & ~ & 88.49 & 88.40 \\
Li \etal~\cite{li2021rotation} & ~ & 89.40 & \textbf{89.32} \\
\hline
\textbf{Ours(PRIN)} & \multirow{2}*{\tabincell{c}{Point XYZ}} &  79.76 & 72.43\\
\textbf{Ours(SPRIN)} & ~ &  86.01 & 86.13\\
\hline
SHOT+PointNet~\cite{shot} & \multirow{2}*{\tabincell{c}{Local Features}}& 48.79 & 48.79\\
CGF+PointNet~\cite{khoury2017learning} & ~ & 57.70 & 57.89\\
\bottomrule[1pt]
\end{tabularx}
\end{center}
\caption{\textbf{Classification results on ModelNet40 dataset.} Performance is evaluated in average instance accuracy. NR means to train with no rotations and test with no rotations. AR means to train with no rotations and test with arbitrary rotations.}
\label{tab:classify}
\end{table}

\begin{table}[h]
\begin{center}
\begin{tabularx}{\linewidth}{p{2.5cm}|p{2.4cm}<{\centering}|p{1.6cm}<{\centering}|p{1.6cm}<{\centering}}
\toprule[1pt]
Method & Input & NR & AR \\
\hline
PointNet~\cite{8099499} & \multirow{12}*{\tabincell{c}{Point XYZ}} & 73.32 & 21.35 \\
PointNet++~\cite{qi2017pointnet++} & ~ & 82.31 & 18.61 \\
SO-Net~\cite{li2018so} & ~ & 86.18 & 13.26\\
SPHNet~\cite{poulenard2019effective} & ~ & 67.43 & 66.52 \\
SRINet~\cite{sun2019srinet} & ~ & 69.53 & 69.18 \\ 
RIConv~\cite{zhang2019rotation} & ~ & 69.32 & 68.44 \\
Kim \etal~\cite{kim2020rotation} & ~ & 68.76 & 68.65 \\
Li \etal~\cite{li2021rotation} & ~ & \textbf{73.42} & \textbf{73.44} \\
\hline
\textbf{Ours(PRIN)} & \multirow{2}*{\tabincell{c}{Point XYZ}} &  58.43 & 52.14\\
\textbf{Ours(SPRIN)} & ~ & 70.12 & 69.83\\
\hline
SHOT+PointNet~\cite{shot} & \multirow{2}*{\tabincell{c}{Local Features}}& 7.43 & 7.43\\
CGF+PointNet~\cite{khoury2017learning} & ~ & 2.87 & 2.87\\
\bottomrule[1pt]
\end{tabularx}
\end{center}
\caption{\textbf{Classification results on ScanObjectNN dataset.} Performance is evaluated in average instance accuracy.}
\label{tab:classifyreal}
\end{table}

\subsubsection{Evaluation Results}
Though classification does not require point-wise rotation invariant features but a global feature, our network still benefits from DAAS and SVC.

We compare PRIN/SPRIN with several state-of-the-art deep learning methods that take point Euclidean coordinates and local rotation invariant features as the input. All the results of baselines are either taken from their original papers or, if not available, reproduced with their released code. RIConv~\cite{zhang2019rotation}, Kim \etal~\cite{kim2020rotation} and Li \etal~\cite{li2021rotation} reported their results on arbitrary rotated test shapes with z-axis data augmentation at training time. In order to make a fair comparison, we re-trained and re-evaluated their models with no rotation augmentation at training time.

Local rotation invariant features like CGF and SHOT are inferior to our method. The results are shown in Table~\ref{tab:classify} and Table~\ref{tab:classifyreal}. Almost all other deep learning methods fail to generalize to unseen orientations, except for a few recent rotation invariant models. Our model achieves competitive results with current state-of-the-arts, on both ModelNet40 and ScanObjectNN. 


\subsection{Application}

\begin{figure}[h]
\begin{center}
  \includegraphics[width=\linewidth]{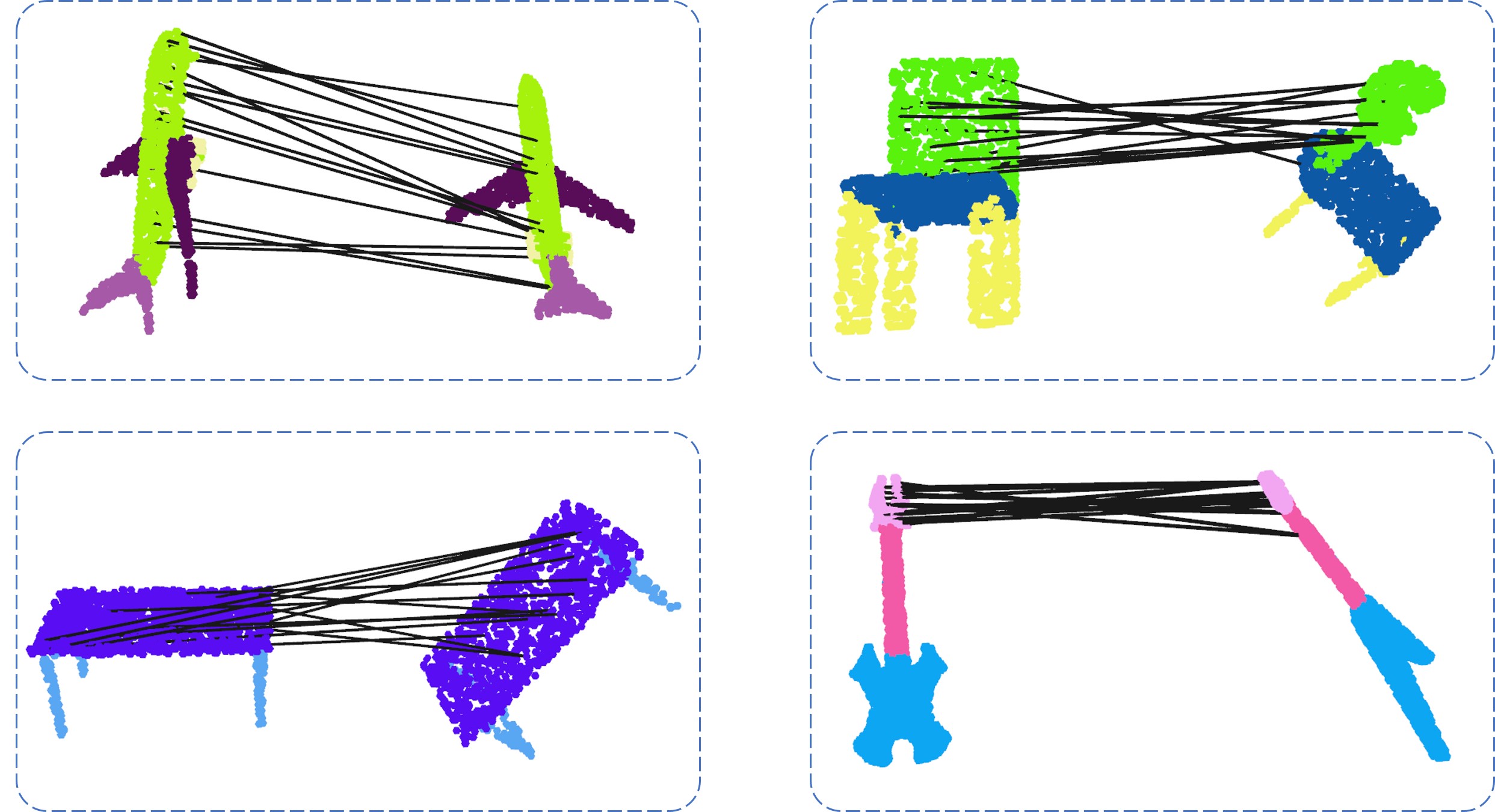}
\end{center}
\caption{\textbf{3D point matching results for PRIN.} Point matching results between the query object features (left) and the retrieved ones (right) under different orientations.}
\label{fig:matching}
\end{figure}

\noindent\textbf{3D Rotation Invariant Point Descriptors.}
For 3D objects, SHOT, CGF and other local features exploits local reference frames (LRF) as local descriptors. Similarly, PRIN/SPRIN is able to produce high quality rotation invariant 3D point descriptors, which is pretty useful as pairwise matching become possible regardless of rotations. In this application, we retrieve the closest descriptor from similar objects under arbitrary orientations. After that, point-level matching is conducted based on these rotation invariant point descriptors. A demonstration is shown in Figure \ref{fig:matching}. 
Through this matching process, we can find out where each point of the query object locates on another object. Moreover, such 3D point descriptors could have the potentiality to do scene searching and parsing as the degree of freedom reduces from six to three, leaving only translations.

To test the matching accuracy on ShapeNet part dataset, we compute and store point-wise features of 80\% test objects (non-rotated) in ShapeNet as a database. Then we compute point-wise features of the other 20\% test objects that are randomly rotated as queries. Feature retrievals are done by finding queries' nearest neighbors in the database. We evaluate against three different baselines to see if points corresponding to the same part are matched. The results are summarized in Table~\ref{tab:desc}. Both PRIN and SPRIN outperform baseline methods by a large margin.
\begin{table}[h]
\begin{center}
\begin{tabular}{l|c}
\hline
Method & Matching Accuracy \\
\hline
PointNet~\cite{8099499} & 38.91\\
SHOT~\cite{shot} & 17.11 \\
CGF~\cite{khoury2017learning} & 52.51\\
\hline
\textbf{Ours(PRIN)} & \textbf{86.12} \\
\textbf{Ours(SPRIN)} & \textbf{78.92} \\
\hline
\end{tabular}
\end{center}
\caption{\textbf{3D descriptor matching results for various methods.} Accuracy is the number of matched points corresponding to the same part divided by the number of all matched points.}
\label{tab:desc}
\end{table}

\begin{figure}[h]
\begin{center}
  \includegraphics[width=\linewidth]{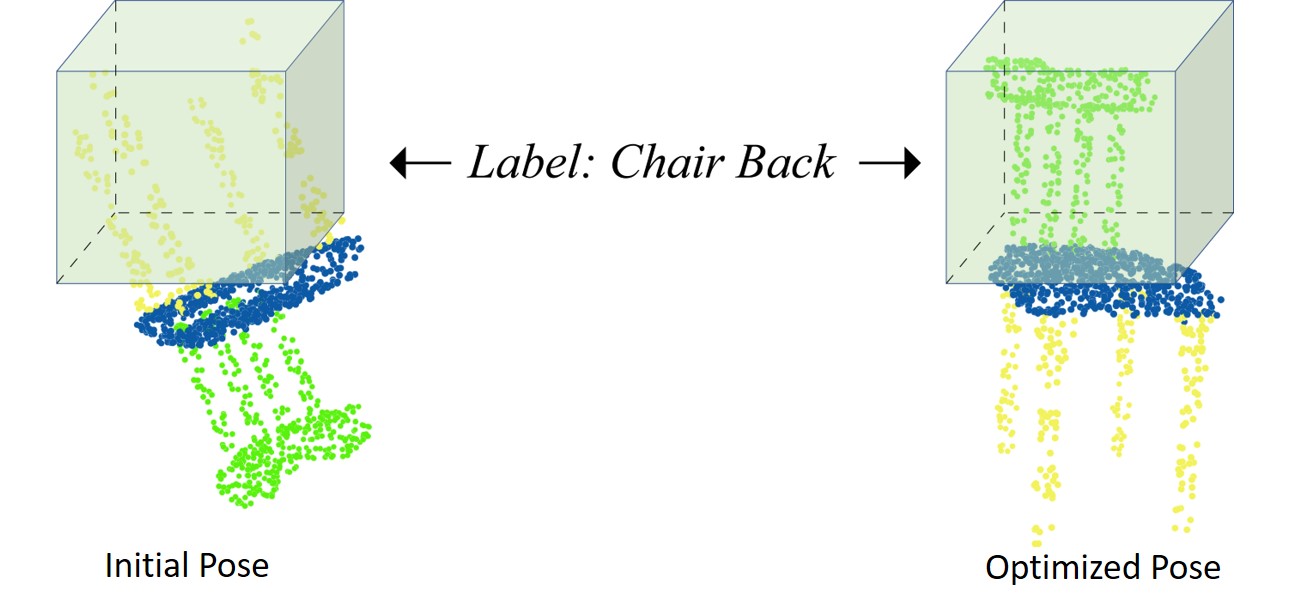}
\end{center}
\caption{\textbf{Chair alignment with its back on the top.} \textbf{Left:} A misalignment induces large KL divergence. \textbf{Right:} Required labels fulfilled with small KL divergence.}
\label{fig:alignment}
\end{figure}

\noindent\textbf{Shape Alignment with Label Priors.}
We now introduce a task that, given some label prior in the Euclidean space, the goal is to align the point cloud to satisfy this requirement. For example, one may align a chair so that its back is on the top. So we add virtual points describing the label requirement. Once the KL divergence between the predicted scores and the ground-truth labels of these virtual points is minimized, the chair is aligned with its back on the top. This is demonstrated in Figure \ref{fig:alignment}.

\subsection{Ablation Study}

In this section, we evaluate numerous variations of our method to verify the rationality of network design. Experiment results are summarized in Table \ref{tab:ablation}, \ref{tab:robust} and Figure \ref{fig:robustness}.
\begin{table}[h]
\begin{center}
\begin{tabularx}{\linewidth}{p{2.5cm}<{\centering}|p{2.5cm}<{\centering}|p{3cm}<{\centering}}
\toprule[1pt]
\tabincell{c}{Bandwidth}  & DAAS & PRIN Acc./mIoU \\
\hline
\textbf{32} & \textbf{Yes} & \textbf{85.99/71.20} \\
16 & Yes & 84.84/68.09 \\
8 & Yes & 82.58/65.43 \\
4 & Yes & 72.20/52.98 \\
\hline
32 & No & 83.13/67.47 \\
\bottomrule[1pt]
\end{tabularx}
\end{center}
\caption{\textbf{Ablation study.} Test accuracy/mIoU of PRIN on the rotated ShapeNet part dataset. Models with various bandwidths and sampling strategies are tested.}
\label{tab:ablation}
\end{table}

\subsubsection{Network Bandwidth} One decisive factor of our network is the bandwidth. Bandwidth is used to describe the sphere precision, which is also the resolution on $S^2\times H$. Mostly, large bandwidth offers more details of spherical voxels, such that our network can extract more specific point features of point clouds. While large bandwidth assures more specific representation of part knowledge, more memory cost is accompanied. Increasing bandwidth from 4 to 32 leads to a relative improvement of 16.04\% and 25.59\% on accuracy and mIoU, which is shown in Table \ref{tab:ablation}.

\subsubsection{DAAS v.s. Uniform Sampling} Recall that in Equation~\ref{eq:change}, we construct our signals on each spherical voxel with a density aware sampling factor. We now compare it with a baseline where uniform sampling is applied and results are given in the last row of Table \ref{tab:ablation}. We see that using the $sin(\beta)$ corrected sampling filter gives superior performance, which confirms our theory.

We also compare extracted features of DAAS and uniform sampling qualitatively in Figure~\ref{fig:feature}. After random rotations of the input canonical point cloud, our DAAS sampled signal obtains much better rotation equivariance than the uniform sampled signal.

\begin{figure}[h]
\begin{center}
   \includegraphics[width=0.95\linewidth]{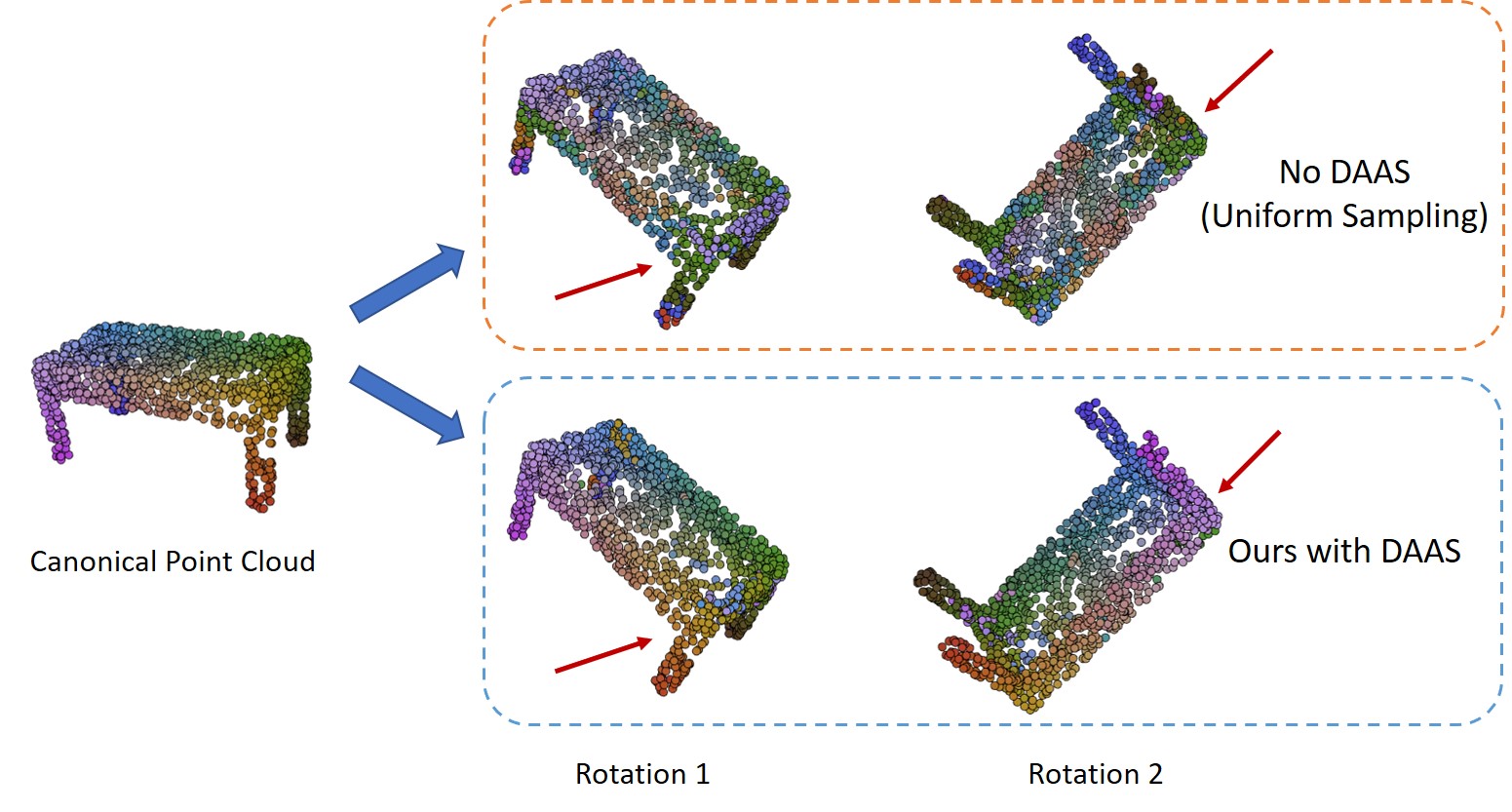}
\end{center}
   \caption{\textbf{Feature visualization of rotated point clouds.} The input point cloud is rotated twice and sampled with DAAS or uniform sampling for further computation. Colors indicate correspondence in the feature space. Our DAAS sampled signal achieves much better correspondence than uniform sampling.}
\label{fig:feature}
\end{figure}

\begin{figure}[h]
\begin{center}
   \includegraphics[width=0.95\linewidth]{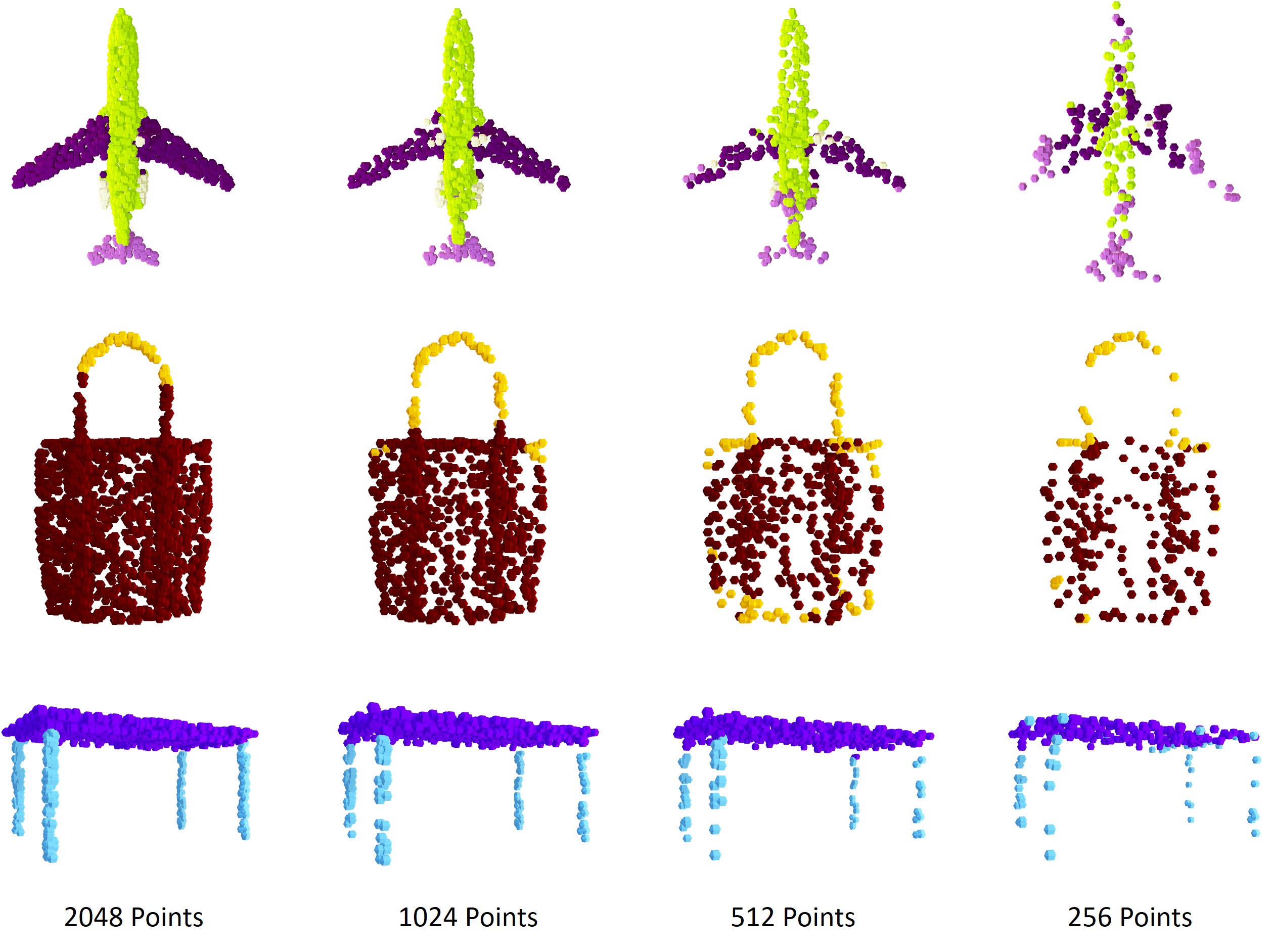}
\end{center}
   \caption{\textbf{Segmentation robustness visualization for PRIN.} \textbf{From left to right}: we sample a subset of 2048, 1024, 512, 256 points from test point clouds respectively. We observe that our method is robust to missing points and gives consistent results.}
\label{fig:robustness}
\end{figure}

\begin{table}[h]
\begin{center}
\begin{tabularx}{\linewidth}{p{2.5cm}<{\centering}|p{2.5cm}<{\centering}|p{3cm}<{\centering}}
\toprule[1pt]
Correlation Features  & Triangle Features & SPRIN Acc./mIoU \\
\hline
\checkmark & ~ & 90.75/79.70 \\
~ & \checkmark & 91.59/81.38 \\
\checkmark & \checkmark & \textbf{92.08/82.67} \\
\bottomrule[1pt]
\end{tabularx}
\end{center}
\caption{\textbf{Ablation study on the selection of rotation invariant features.} Test accuracy/mIoU of SPRIN on rotated ShapeNet part dataset.}
\label{tab:feature}
\end{table}

\begin{table}[h]
\begin{center}
\begin{tabularx}{\linewidth}{p{2.5cm}<{\centering}|p{2.5cm}<{\centering}|p{3cm}<{\centering}}
\toprule[1pt]
\tabincell{c}{\# of points}  & PRIN Acc./mIoU & SPRIN Acc./mIoU \\
\hline
2048 & \textbf{85.99/71.20} & \textbf{92.08/82.67} \\
1024 & 84.19/67.43 & 90.92/79.79 \\
512 & 71.19/51.50 & 85.59/72.04 \\
256 & 55.23/36.80 & 64.17/51.38 \\
\bottomrule[1pt]
\end{tabularx}
\end{center}
\caption{\textbf{Qualitative results for segmentation robustness.} Test accuracy/mIoU of PRIN/SPRIN on rotated ShapeNet part dataset. Various number of points are sub-sampled.}
\label{tab:robust}
\end{table}
\subsubsection{Selection of Rotation Invariant Features}
In SPRIN, we use the correlation feature $\mathcal{T}(x_j)^{-1}x_i$ and six invariant features of the triangle formed by $x_i$, $x_j$ and $\frac{1}{N}\sum_i^Nx_i$. We explore whether these features contribute to the final segmentation. Results are given in Table~\ref{tab:feature}.
\subsubsection{Segmentation Robustness}
PRIN and SPRIN also demonstrate some robustness to corrupted and missing points. Although the density of point clouds declines, our network still segments correctly for each point. Qualitative results are given in Table~\ref{tab:robust}. Both PRIN and SPRIN maintain a good accuracy up to 4x down-sampling (512 points). In Figure \ref{fig:robustness}, we can see that PRIN predicts consistent labels regardless of the point density. 

\section{Conclusion}

We present PRIN, a network that takes any input point cloud and leverages Density Aware Adaptive Sampling (DAAS) to construct signals on spherical voxels. Then Spherical Voxel Convolution (SVC) and Point Re-sampling follow to extract point-wise rotation invariant features. We place two different output heads to do both 3D point clouds classification and part segmentation. In addition, we extend PRIN to a sparse version called SPRIN, which directly operates on sparse point clouds. Our experiments show that both PRIN and SPRIN are robust to arbitrary orientations. Our network can be applied to 3D point feature matching and shape alignment with label priors. We show that our model can naturally handle arbitrary input orientations and provide detailed theoretical analysis to help understand our network.

\ifCLASSOPTIONcompsoc
  \section*{Acknowledgments}
\else
  \section*{Acknowledgment}
\fi
This work is supported in part by the National Key R\&D Program of China, No. 2017YFA0700800 and No. 2019YFC1521104, SHEITC (2018-RGZN-02046) and Shanghai Qi Zhi Institute. This work was also supported by the National Natural Science Foundation of China under Grant 51675342, Grant 51975350, Grant 61972157 and Grants 61772332.

\bibliographystyle{IEEEtran} \bibliography{main}

\begin{IEEEbiography}
    [{\includegraphics[width=1in,height=1.25in,clip,keepaspectratio]{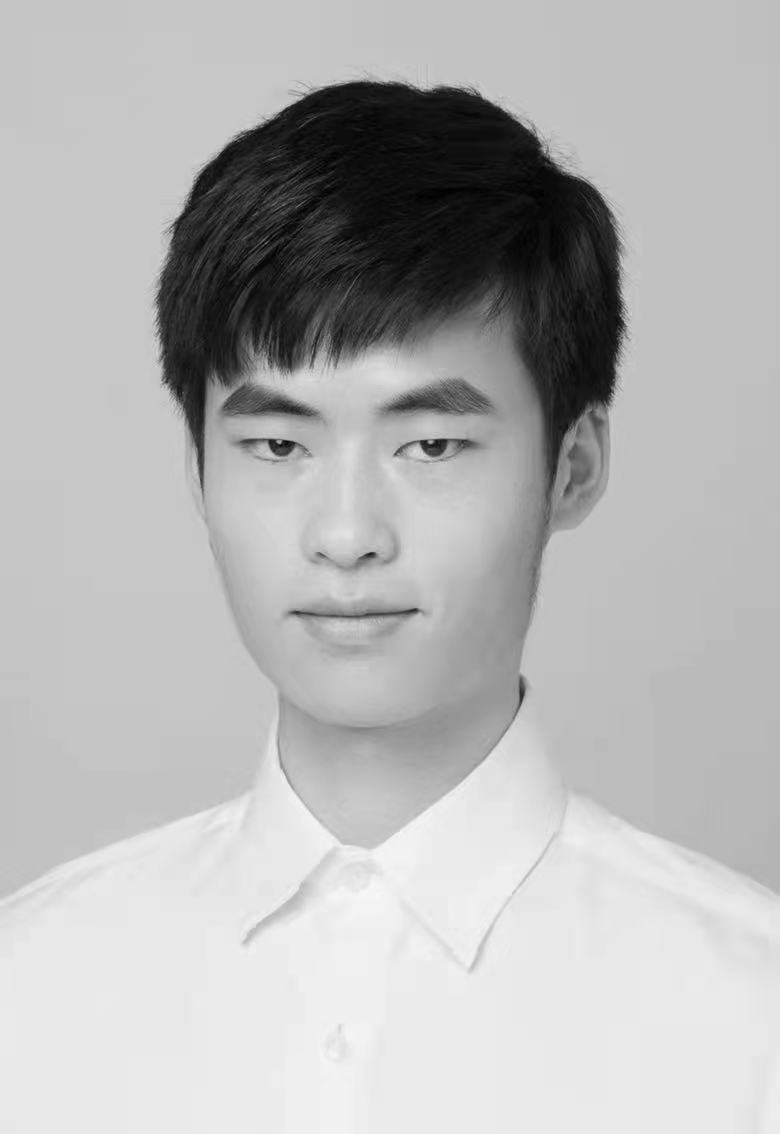}}]{Yang You}
received the BS and MS degrees from Shanghai Jiao Tong University and University of Virginia in 2016 and 2017 respectively. He is now a third-year student pursuing his doctoral degree at Mechanical Engineering department in Shanghai Jiao Tong University. His main interests lie in 3D Computer Vision,
Computer Graphics and Robotics.
\end{IEEEbiography}

\begin{IEEEbiography}
    [{\includegraphics[width=1in,height=1.25in,clip,keepaspectratio]{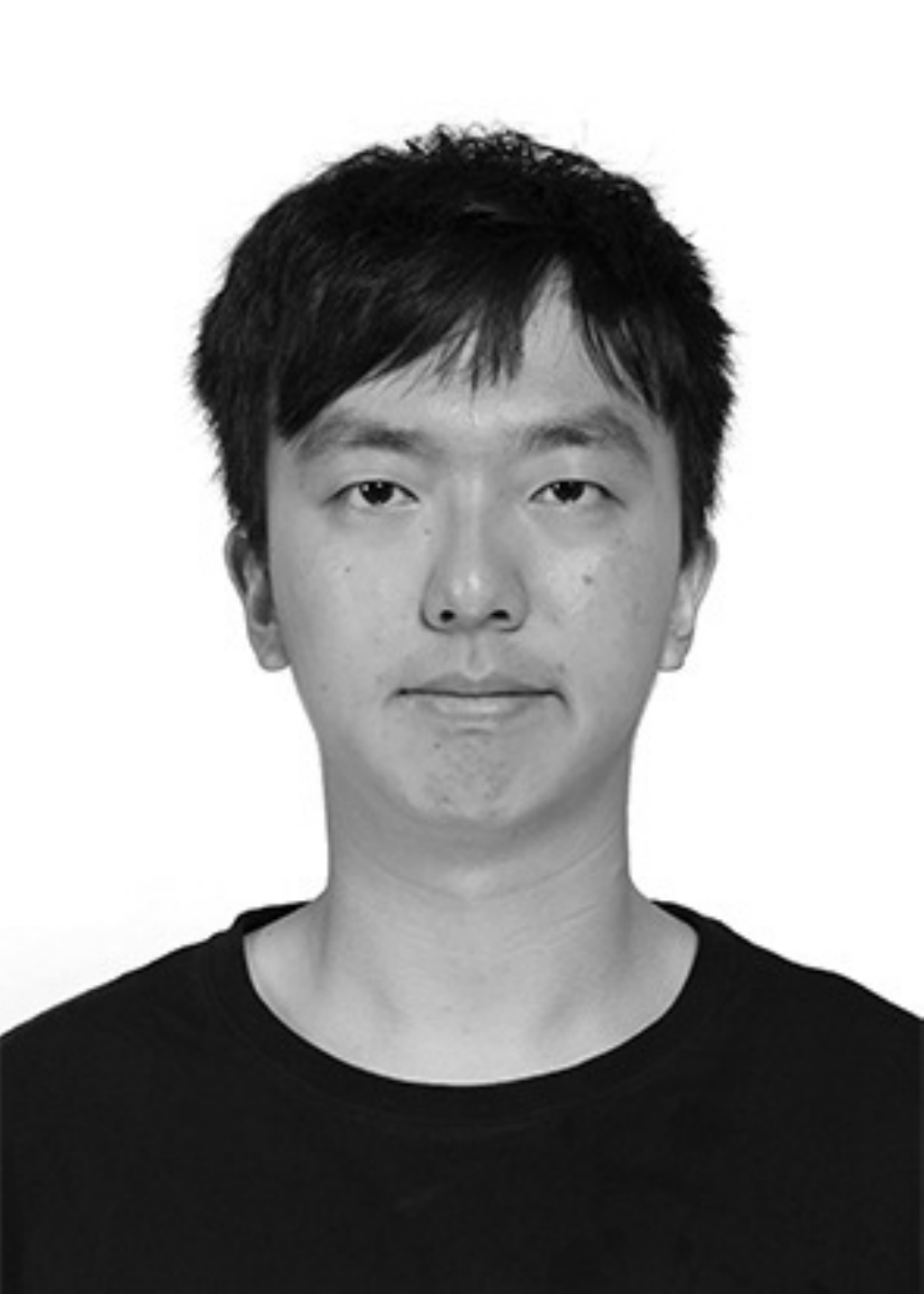}}]{Yujing Lou} 
received the B.S. degree in computer science and technology from Harbin Institute of Technology, China, in 2018 and the M.S. degree in computer science and technology from Shanghai Jiao Tong University, China, in 2020. He is currently a Ph.D. student with MVIG lab, Shanghai Jiao Tong University. His research interests include 3D scene/object perception and robot vision.
\end{IEEEbiography}

\begin{IEEEbiography}
    [{\includegraphics[width=1in,height=1.25in,clip,keepaspectratio]{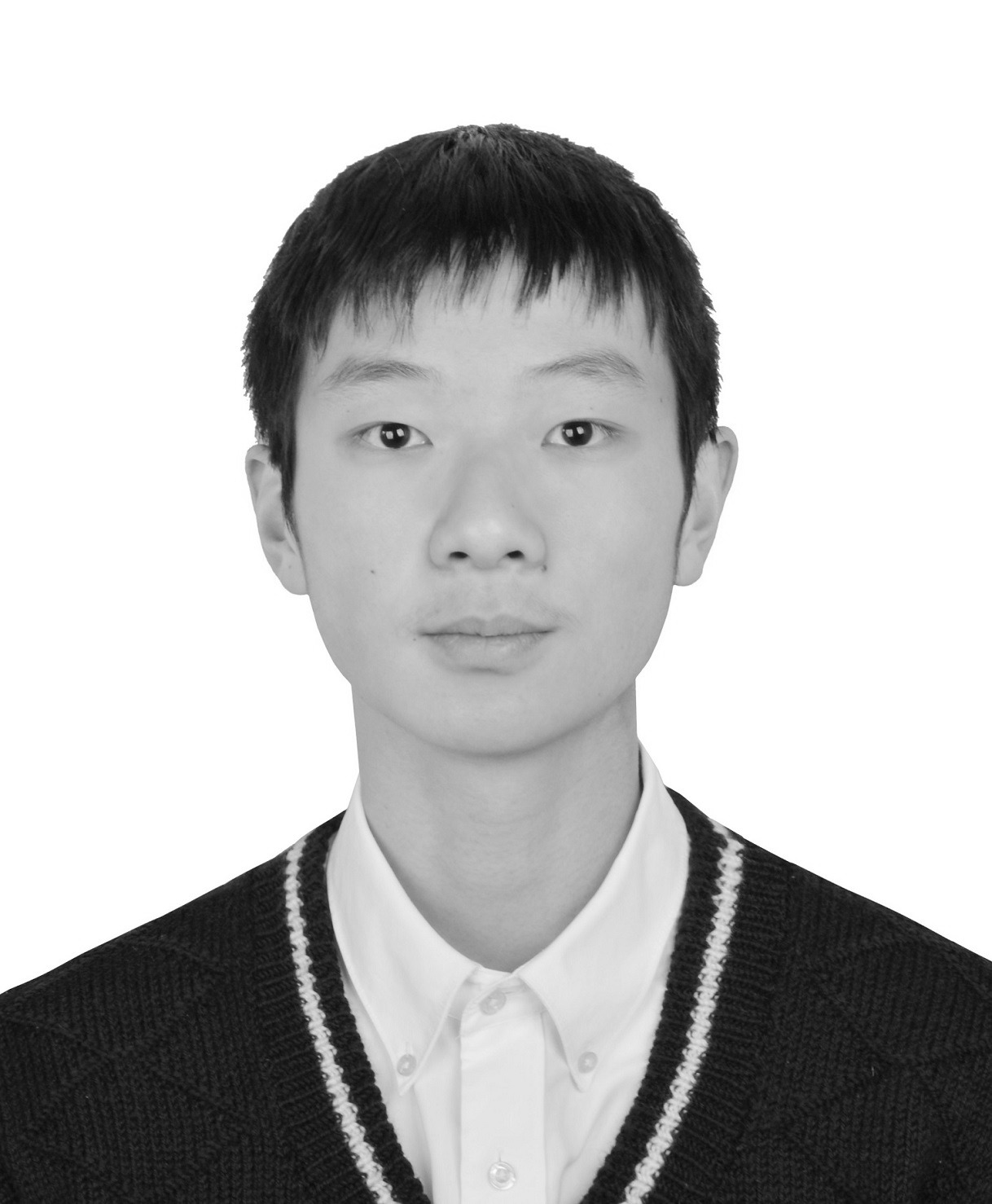}}]{Ruoxi Shi} 
is currently an undergraduate student in School of Electronic Information and Electrical Engineering at Shanghai Jiao Tong University, majoring in Artificial Intelligence. He has been working with Prof. Cewu Lu and Dr. Yang You at the Machine Vision and Intelligence Group since 2020. 
He is currently interested in 3D computer vision and deep learning.
\end{IEEEbiography}

\begin{IEEEbiography}
    [{\includegraphics[width=1in,height=1.25in,clip,keepaspectratio]{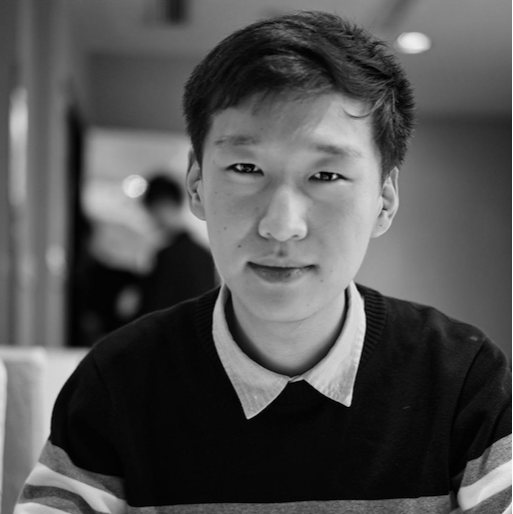}}]{Qi Liu} 
received his B.S. degrees from Dalian University of Technology in 2016. He is now a PhD candidate in Shanghai Jiao Tong University. His research interests include 3D computer vision and computer graphics.
\end{IEEEbiography}

\begin{IEEEbiography}
    [{\includegraphics[width=1in,height=1.25in,clip,keepaspectratio]{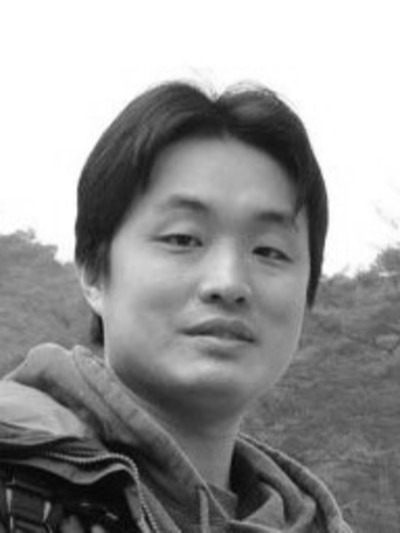}}]{Yu-Wing Tai} 
received his B.Eng. (First Class Honors) and M Phil degrees from the Department of Computer Science and Engineering, HKUST in 2003 and 2005 and PhD degree from the National University of Singapore in 2009. He is currently an Adjunct Professor at the Department of Computer Science and Engineering, HKUST. He is a research director of YouTu research lab of Social Network Group of Tencent. He was a principle research scientist of SenseTime Group Limited from September 2015 to December 2016. He was an associate professor at the Korea Advanced Institute of Science and Technology (KAIST) from July 2009 to August 2015. From Sept 2007 to June 2008, he worked as a full-time student internship in the Microsoft Research Asia (MSRA). He was awarded the Microsoft Research Asia Fellowship in 2007, and the KAIST 40th Anniversary Academic Award for Excellent Professor in 2011 respectively. His research interests include deep learning, computer vision and image/video processing.
\end{IEEEbiography}

\begin{IEEEbiography}
    [{\includegraphics[width=1in,height=1.25in,clip,keepaspectratio]{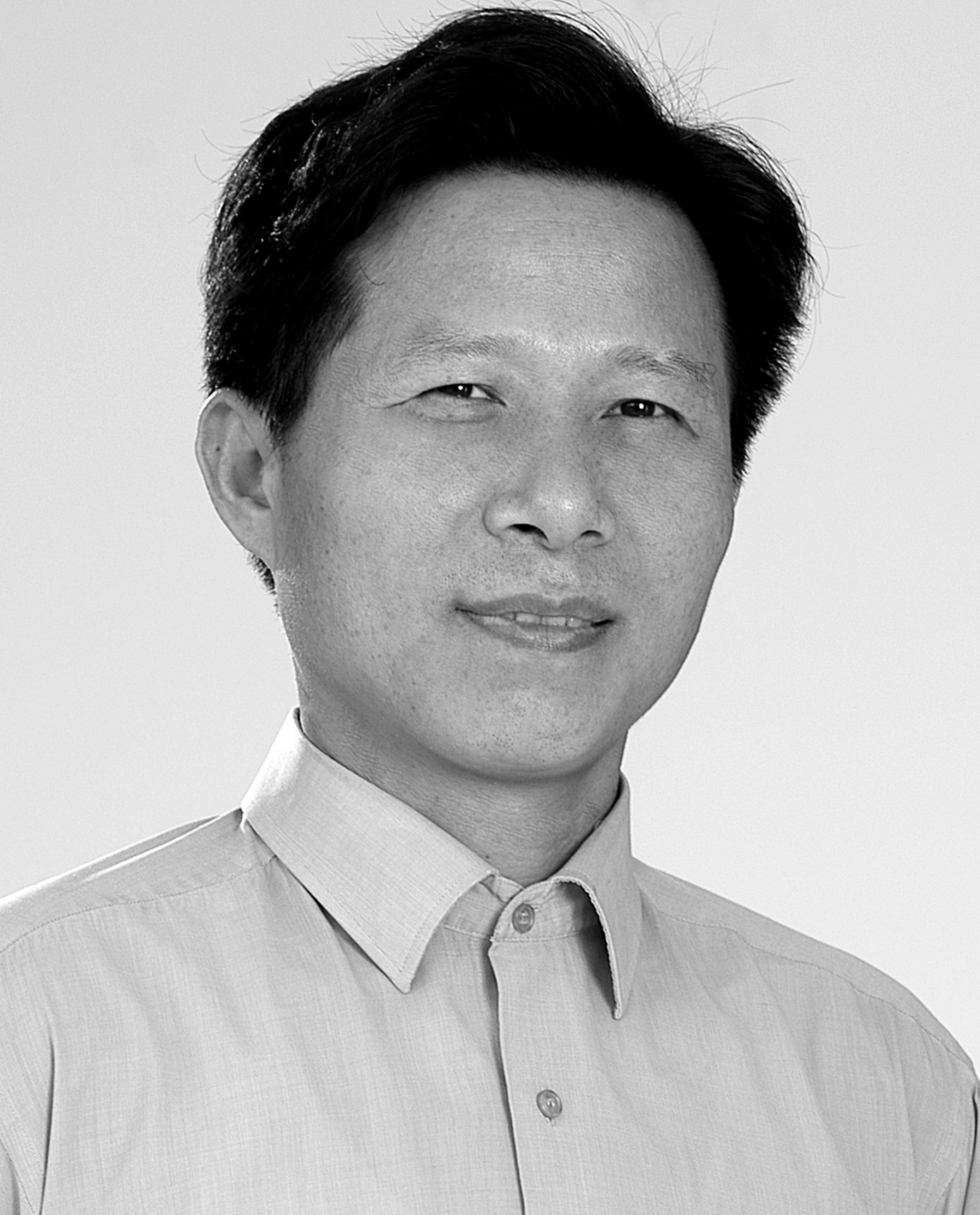}}]{Lizhuang Ma}
received his B.S. and Ph.D. degrees from the Zhejiang University, China in 1985 and 1991, respectively. He is now a Distinguished Professor, at the Department of Computer Science and Engineering, Shanghai Jiao Tong University, China and the School of Computer Science and Technology, East China Normal University, China. He was a Visiting Professor at the Frounhofer IGD, Darmstadt, Germany in 1998, and a Visiting Professor at the Center for Advanced Media Technology, Nanyang Technological University, Singapore from 1999 to 2000. His research interests include computer vision, computer aided geometric design, computer graphics, scientific data visualization,
computer animation, digital media technology, and theory and applications for computer graphics, CAD/CAM.
\end{IEEEbiography}

\begin{IEEEbiography}
    [{\includegraphics[width=1in,height=1.25in,clip,keepaspectratio]{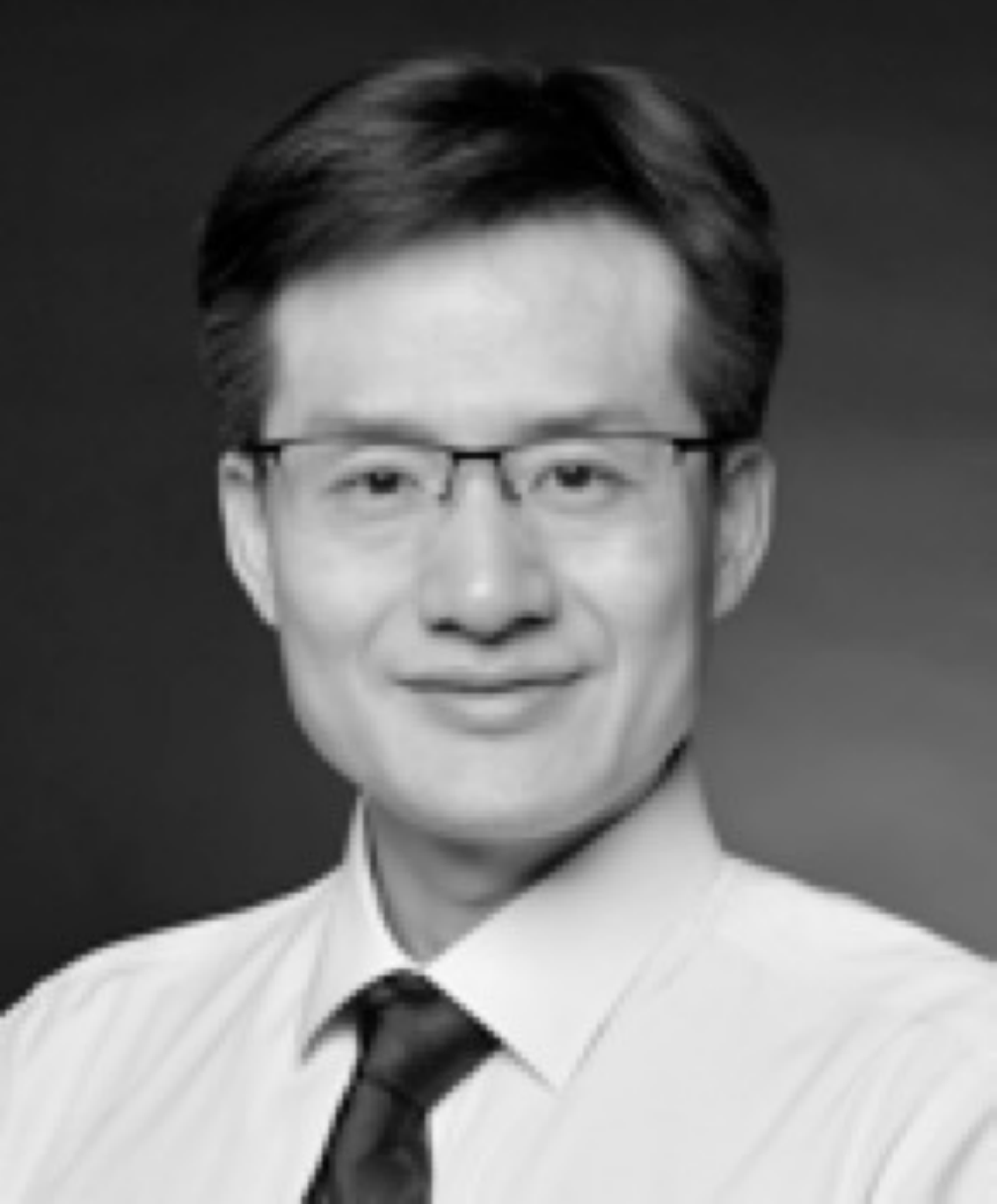}}]{Weiming Wang}
is a Professor with the School of Mechanical Engineering, Shanghai Jiao Tong University, Shanghai, China. His research interests include machine vision, ﬂexible robots, and human-robot interaction.
\end{IEEEbiography}

\begin{IEEEbiography}
    [{\includegraphics[width=1in,height=1.25in,clip,keepaspectratio]{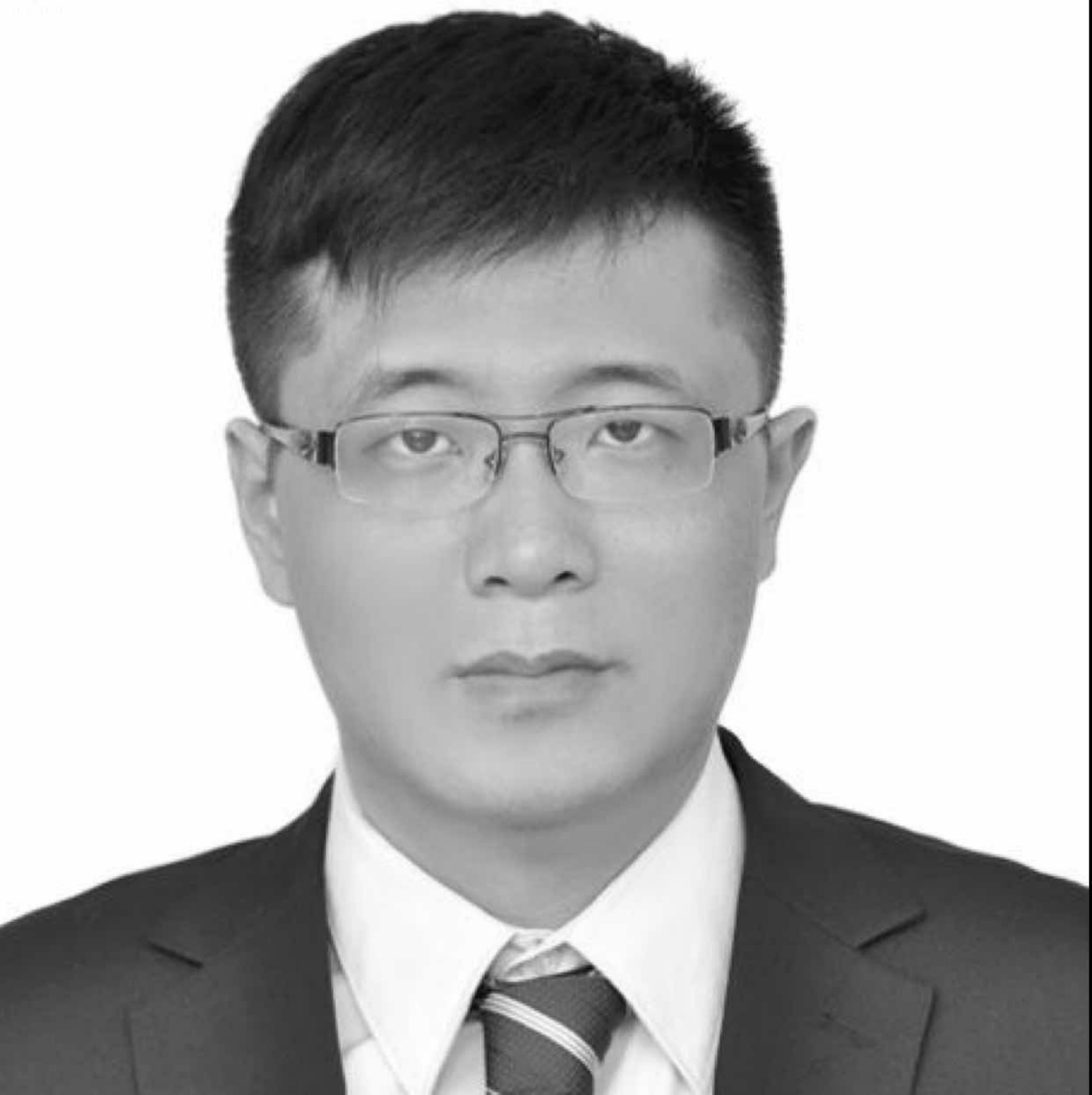}}]{Cewu Lu}
is an Associate Professor at Shanghai Jiao Tong University (SJTU). Before he joined SJTU, he was a research fellow at Stanford University, working under Prof. Fei-Fei Li and Prof. Leonidas J. Guibas. He was a Research Assistant Professor at Hong Kong University of Science and Technology with Prof. Chi Keung Tang. He got his PhD degree from the Chinese Univeristy of Hong Kong, supervised by Prof. Jiaya Jia. He is one of the core technique members in Stanford-Toyota autonomous car project. He serves as an associate editor for Journal CVPR and reviewer for Journal TPAMI and IJCV. His research interests fall mainly in computer vision, deep learning, deep reinforcement learning and robotics vision.
\end{IEEEbiography}

\vfill

\end{document}